%% file: main.tex
\documentclass[10pt,journal,compsoc]{IEEEtran}

\ifCLASSOPTIONcompsoc
  \usepackage[nocompress]{cite}
\else
  \usepackage{cite}
\fi

\usepackage{soul}
\usepackage{url}
\usepackage[hidelinks]{hyperref}
\usepackage[utf8]{inputenc}
\usepackage{graphicx}
\usepackage{amsmath}
\usepackage{amsthm}
\usepackage{booktabs}
\usepackage{algorithm}
\usepackage{algorithmic}
\usepackage{array}

\newtheorem{theorem}{Theorem}

\usepackage{subfiles} 
\usepackage{amssymb,amsfonts}
\usepackage{textcomp}
\usepackage{xcolor}
\usepackage{mathtools}
\usepackage[switch]{lineno}
\usepackage{subfigure}
\usepackage{array}

\newtheorem{lemma}{Lemma}
\newtheorem{prop}{Proposition}
\newtheorem{definition}{Definition}

\begin{document}

\title{Federated Learning with Nesterov Accelerated Gradient}

\author{Zhengjie~Yang, 
Wei~Bao, 
Dong~Yuan,
Nguyen~H.~Tran,
and~Albert~Y.~Zomaya%
}

\IEEEtitleabstractindextext{%
\begin{abstract}

Federated learning (FL) is a fast-developing technique that allows multiple workers to train a global model based on a distributed dataset.  
Conventional FL (FedAvg) employs gradient descent algorithm, which may not be efficient enough. Momentum is able to improve the situation by adding an additional momentum step to accelerate the convergence and has demonstrated its benefits in both centralized and FL environments. It is well-known that Nesterov Accelerated Gradient (NAG) is a more advantageous form of momentum, but it is not clear how to quantify the
benefits of NAG in FL so far. This motives us to propose FedNAG, which employs NAG in each worker as well as NAG momentum and model aggregation in the aggregator. We provide a detailed convergence analysis of FedNAG and compare it with FedAvg. Extensive experiments based on real-world datasets and trace-driven simulation are conducted, demonstrating that FedNAG increases the learning accuracy by 3--24\% and decreases the total training time by 11--70\% compared with the benchmarks under a wide range of settings. 
\end{abstract}

}

\maketitle
\thispagestyle{empty}
\pagestyle{empty}

\IEEEdisplaynontitleabstractindextext
 
\IEEEpeerreviewmaketitle

\IEEEraisesectionheading{\section{Introduction}\label{sec:introduction}}

\IEEEPARstart{W}{ith} the advancement of Internet of Things (IoT), Industry 4.0, and Artificial Intelligence, machine learning applications such as image classification~\cite{lu2007survey}, automatic driving~\cite{naranjo2005power}, and Automatic Speech Recognition (ASR)~\cite{yu2016automatic} are rapidly developed. Since the tremendous machine learning data are distributed in individual users, conventional centralized machine learning is insufficient when a large volume and sensitive data are required to be uploaded to remote data-centers. Moreover, in many situations, the individual users are not willing to share their sensitive raw data so it is infeasible to implement centralized machine learning. To address the issue, Federated Learning (FL) emerges \cite{mcmahan2017communication}. It allows individual users to participate in the global model training without sharing their raw data.

Mobile Edge Computing (MEC), as shown in Fig.~\ref{fig:edge}, is a perfect venue to implement FL~\cite{lim2020federated}, where multiple edge devices are distributed and connected to an edge server. Workers can be any form of edge devices such as laptop, smartphone, tablet, IoT devices, etc. The edge server can aggregate the local models from edge workers and re-distribute the global model back to edge workers. Since edge workers have limited computation and communication capacities, we need to develop more efficient algorithms to accelerate the convergence and finally decrease the communication and computing workload and total training time.

One commonly adopted FL algorithm is FedAvg \cite{mcmahan2017communication}, which performs gradient descent \cite{ruder2016overview} at each worker: Each worker locally updates its weights by gradient descent for a number of local iterations by its local dataset, and then the aggregator averages the weights from all workers and distribute them to the workers again.  The above process is repeated for multiple rounds. (One round consists of a number of local iterations and one communication step between workers and the aggregator.) However, one disadvantage of gradient descent is its low efficiency for convergence and potential in oscillations \cite{goh2017why,ruder2016overview}. The weight update at the specific iteration $t$ is only governed by the current gradient at this iteration. It does not consider the past weight update steps. Momentum~\cite{polyak1964some} is able to improve the situation by adding an additional momentum step to include the difference between past and current weights on the basis of gradient descent step.  
The advantage of momentum has been well studied in centralized setting\cite{vaswani2019fast,yan2018unified,assran2020convergence}. It also motivates researchers to employ momentum in FL environment \cite{huo2020faster,wang2019slowmo,karimireddy2020mime}. Depending on where the momentum is adopted, these works can be categorized as aggregator momentum and worker momentum. The aggregator momentum applies momentum at the aggregator but it can only utilize the momentum acceleration 
every $\tau$ local iterations when global aggregation happens. ($\tau$ is the number of local iterations between two global aggregations.) 
Worker momentum utilizes the momentum acceleration at the worker but is influenced by out-of-date momentum since the momentum is only updated when global aggregation occurs and each worker does not update it in each local iteration.

\begin{figure}[tb!]
    \centering
    \includegraphics[width=\linewidth]{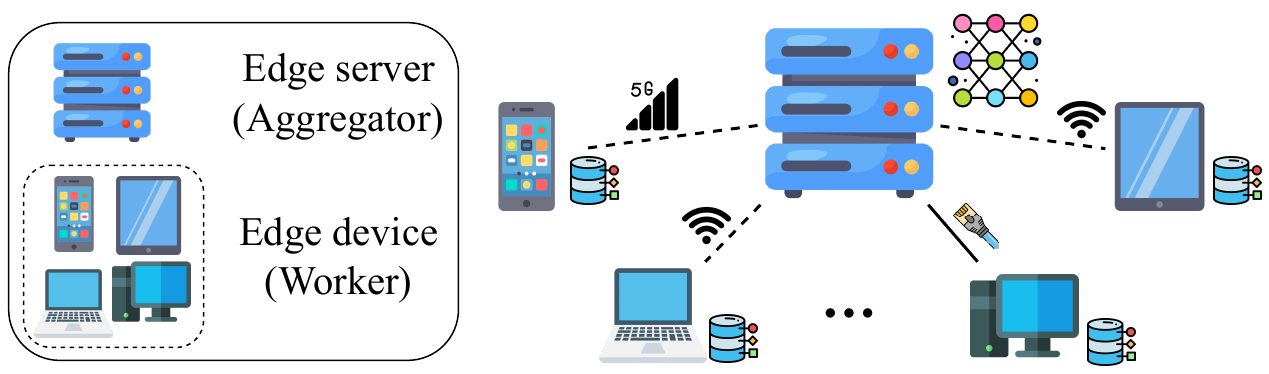}
    \caption{Federated Learning in Mobile Edge Computing (MEC).}
    \label{fig:edge}
\end{figure}

Nesterov Accelerated Gradient (NAG) \cite{NAG} is known to be an advantageous form of momentum~\cite{polyak1964some}. Existing works have demonstrated the advantage of momentum in both centralized and FL environment, but it is not clear how to quantify the benefits of NAG in FL in the literature. This motives us to propose an NAG style FL, namely FedNAG: (1) Each worker locally updates its weights and momenta using NAG for $\tau$ iterations on its local dataset;  (2) the aggregator collects and averages the weights and  momenta from all workers and distribute them to the workers again; (1) and (2) are repeated for multiple rounds until the training loss is sufficiently small. 

We theoretically provide a detailed convergence analysis for FedNAG. The progress mainly includes three steps: (1) We define virtual update as if centralized NAG is conducted between two global aggregations; (2) We bound the gap of weights $\mathbf{w}$ between FedNAG update and virtual update; and (3) We bound the values of global loss functions $F(\mathbf{w})$  between FedNAG and the optimal solution. Since the convergence analysis of FedAvg is provided in~\cite{wang2019adaptive}, we compare the convergence performance of FedNAG and FedAvg and derive the conditions that FedNAG outperforms FedAvg.

Experimentally, we use different models such as linear regression, logistic regression, CNN, and DNN based on MNIST, CIFAR-10, and CIFAR-100 datasets, to test the performance of FedNAG. We analyze the impact of different factors such as number of workers, number of local updates between two global aggregataion $\tau$, and momentum coefficient. We also conduct a trace-driven simulation to emulate a MEC environment to test the real-world total training time consisting of computation delay and communication delay. The experiment shows that FedNAG increases the learning accuracy by 3--24\% and decreases the total training time by 11--70\% compared with FedMom~\cite{huo2020faster} 
and FedAvg under a wide range of settings.

\section{Related works}

\input{2relatedwork_new}

\section{System Model and Preliminaries}\label{sec:model}

\input{3modelDesign}

\section{Convergence Analysis of FedNAG}

\input{4convergenceAnalysis2}

\section{Experiments} \label{sec:exp}
\input{5experiment}

\section{Conclusion} \label{sec:conclusion}
\input{6conclusion}


\ifCLASSOPTIONcaptionsoff
  \newpage
\fi

\bibliographystyle{IEEEtran}
\bibliography{main}

\newpage
\appendices
\input{appendix}

\end{document}

%% file: 2relatedwork_new.tex
\begin{table}[t]
\centering
\caption{Comparison between different topics in Federated Learning}
\label{tab:related_work}
\begin{tabular}{p{0.2\columnwidth} p{0.68\columnwidth}}
\toprule
Topic & Strategy/Algorithm   \\
\midrule
Proximity & FedProx~\cite{FedProx},  SCAFFOLD~\cite{karimireddy2020scaffold} \\
\midrule
Quantization & FedPAQ~\cite{reisizadeh2020fedpaq},  ACGD~\cite{li2020acceleration}\\
\midrule
Secure FL & NbAFL~\cite{NbAFL}, Privacy-Preserving FL~\cite{PP_FL} \newline  Blockchain-supported FL~\cite{blockchain} \\
\midrule
Vehicular FL  & FVC~\cite{FVC}, FVN~\cite{FVN}\\ 
\midrule
Momentum &  FedMo \cite{huo2020faster}, SlowMo \cite{wang2019slowmo}, Mime \cite{karimireddy2020mime} \\
\bottomrule
\end{tabular}
\end{table}

\begin{table*}[ht]
\centering
\caption{Comparison between different variants of momentum based federated learning algorithm}
\label{tab:compare}
\begin{tabular}{p{0.2\columnwidth} p{0.76\columnwidth} p{0.9\columnwidth}}
\toprule
Algorithm & Local updates & Global updates \\
\midrule
FedNAG & $\mathbf{v}_i(t) \gets \gamma\mathbf{v}_i(t-1) - \eta\nabla F_i(\mathbf{w}_i(t-1))$ \newline $\mathbf{w}_i(t) 
    \gets \mathbf{w}_{i}(t-1) + \gamma\mathbf{v}_{i}(t) - \eta\nabla F_i(\mathbf{w}_{i}(t-1))$
    &
    $\mathbf{v}(t) \gets \frac{\sum_{i=1}^N D_i \mathbf{v}_i(t)}{D}$\newline $\mathbf{w}(t) \gets \frac{\sum_{i=1}^N D_i \mathbf{w}_i(t)}{D}$\\
\midrule
FedMom \cite{huo2020faster}
& $\mathbf{w}_{i}(t) \gets \mathbf{w}_{i}(t-1) - \eta\nabla F_{i}(\mathbf{w}_{i}(t-1))$ & $\mathbf{v}(t) \gets \mathbf{w}(t-\tau) -  \sum_{i=1}^N\frac{ D_i}{D}\left(\mathbf{w}(t-\tau)-\mathbf{w}_i(t)\right)$\newline $\mathbf{w}(t) 
    \gets \mathbf{v}(t) +\gamma(\mathbf{v}(t)-\mathbf{v}(t-\tau))$\\
\midrule
SlowMo \cite{wang2019slowmo} & $\mathbf{w}_{i}(t) \gets \mathbf{w}_{i}(t-1) - \eta\nabla F_{i}(\mathbf{w}_{i}(t-1))$ &
$\mathbf{v}(t) \gets \gamma\mathbf{v}(t-\tau)+ \frac{1}{
\eta}\left(\mathbf{w}(t-\tau)- \sum_{i=1}^N\frac{ D_i}{D}\mathbf{w}_{i}(t)\right) \newline
\mathbf{w}(t)\gets \mathbf{w}(t-\tau)-\eta\mathbf{v}(t)$\\
\midrule
Mime \cite{karimireddy2020mime} & $\mathbf{w}_{i}(t) \gets \mathbf{w}_{i}(t-1) - \newline \eta ((1-\gamma)\nabla F_{i}(\mathbf{w}_{i}(t-1)) + \gamma \mathbf{v}((\lfloor \frac{t}{\tau}\rfloor-1)\tau)$ & $\mathbf{v}(t) \gets (1-\gamma)\nabla F_i(\mathbf{w}(t-\tau))+ \gamma\mathbf{v}(t-\tau)$ \newline
$\mathbf{w}(t) \gets \frac{\sum_{i=1}^N D_i \mathbf{w}_i(t)}{D}$\\

\bottomrule
\end{tabular}
\end{table*}

\subsection{Federated Learning Algorithms and Applications}
Federated Learning was first proposed in FedAvg \cite{mcmahan2017communication}. It allows multiple clients to collaboratively train a global learned model without sharing their raw data. Apart from 
FedAvg, there are many algorithms/strategies that have been studied in FL (e.g., proximity~\cite{FedProx}, quantization~\cite{reisizadeh2020fedpaq}, differential privacy~\cite{DP}, Vehicular FL~\cite{FVN} etc.). 
The idea of proximity is to apply an additional term on the basis of gradient descent by using the information of the global model to correct the local update. FedProx~\cite{FedProx} employs the proximal term to restrict
the local updates to be closer to the global model. SCAFFOLD~\cite{karimireddy2020scaffold} employs the control variate to prevent ``client-drift'' \cite{karimireddy2020scaffold}. The idea of the quantization technique is to reduce the size of transmission payload so as to reduce the communication overhead. FedPAQ~\cite{reisizadeh2020fedpaq} employs quantization operators on
the transmitted massages while ACGD~\cite{li2020acceleration} employs gradient compression for communication. For Secure FL, NbAFL~\cite{NbAFL} and \cite{PP_FL} apply the Differential Privacy (DP)~\cite{DP} in FL environment to protect the information leakage of the original raw data. In \cite{blockchain}, authors proposes an adaptive framework consisting of blockchain~\cite{bodkhe2020blockchain} and Reinforcement Learning (RL)~\cite{RL} in FL to achieve higher trust and security. For Vehicular FL, FVC~\cite{FVC} and FVN~\cite{FVN} extend the traditional vehicular network \cite{FVN} to the FL environment, where the vehicular network consists of groups of moving or stationary vehicles connected by a wireless network. However, these sub-topics are not the focus in this paper. We focus on momentum~\cite{ruder2016overview}. The mainstream FL sub-topics are summarized in Table \ref{tab:related_work}.

\subsection{Momentum in Machine Learning}
Momentum is a method that helps accelerate gradient descent in the relevant direction by adding a fraction $\gamma$ of the difference between past and current model vectors  \cite{ruder2016overview}. There are two typical forms of momentum in the literature~\cite{zhang2021dive,ruder2016overview}. One is Polyak's momentum \cite{polyak1964some} and the update rule is as follows:
\begin{align}
        \mathbf{v}(t) &= \gamma\mathbf{v}(t-1) - \eta\nabla F(\mathbf{w}(t-1)),\label{HB.1}\\
        \mathbf{w}(t) &= \mathbf{w}(t-1) + \mathbf{v}(t),\label{HB.2}
\end{align}
with $\gamma\in[0,1), t=1,2,3..., \mathbf{v}(0)=0$, where $\gamma$ is momentum factor (weight of momentum), $t$ is update iteration, $\mathbf{v}(t)$ is momentum term at iteration $t$, and $\mathbf{w}(t)$ is model parameter at iteration $t$. 
Through this method, the momentum term increases for dimensions whose gradients point in the same directions and reduces updates for dimensions whose gradients change directions. As a result, we gain faster convergence and reduced oscillation \cite{ruder2016overview,goh2017why}.
Another typical form of momentum is Nesterov Accelerate Gradient (NAG) \cite{NAG} and known as a more advantageous version of momentum compared with Polyak's momentum \cite{polyak1964some}. 
NAG calculates the gradient based on an approximation of the next position of parameters, i.e., $\nabla F(\mathbf{w}(t-1)+\gamma\mathbf{v}(t-1))$, instead of  $\nabla F(\mathbf{w}(t-1))$ in Polyak's momentum, leading to better convergence performance. This leads us to employ NAG instead of Polyak's momentum in our proposed FedNAG.

\begin{table}[t]
\centering
\caption{Key Notations}
\label{tab:notation}
\begin{tabular}{p{0.06\columnwidth}p{0.83\columnwidth}}
\hline
$N$ & number of workers\\
$T$ & number of total local (worker) iterations indexed by $t$\\
$K$ & number of global aggregations indexed by $k$\\
$D_i$ & number of samples for local dataset $i$\\
$D$ & total number of samples\\
$\eta$ & learning step size hyper parameter \\
$\gamma$ & momentum hyper parameter \\
$\tau$ & number of local iterations between two global aggregations \\
$F(\mathbf{w})$ & global loss function\\
$F_i(\mathbf{w})$ & local loss function in worker $i$\\
$\mathbf{w}^{\mathrm{f}}$        & practical model parameter that the learning can obtain \\
$\mathbf{w}(t)$         & global model parameter at iteration $t$ \\
$\mathbf{w}_i(t)$ & local model parameter at iteration $t$ in worker $i$ \\
$\mathbf{v}(t)$ &  global momentum parameter at iteration $t$\\
$\mathbf{v}_i(t)$ & local momentum parameter at iteration $t$ in worker $i$\\
\hline
\end{tabular}
\end{table}

\subsection{Momentum in Federated Learning}
Momentum has been already well studied and proved to be more advantageous in centralized machine learning. In \cite{vaswani2019fast}, authors study the utilization of momentum in over-parameterized models. \cite{yan2018unified} provides an unified convergence analysis for both Polyak's  momentum and NAG. \cite{assran2020convergence} studies NAG in stochastic settings. 
It also attracts researchers' attention to apply momentum in FL environment. FedMom \cite{huo2020faster} and SlowMo \cite{wang2019slowmo} perform momentum update in the aggregator only. 
In Mime \cite{karimireddy2020mime}, the aggregator computes the momentum and distributes it to the workers, which is then used by workers for local iterations. 
All these works also demonstrate the benefits of momentum in FL, but with more simplified approach compared with FedNAG. It is well-known that Nesterov Accelerated Gradient (NAG) \cite{NAG} is a more advantageous form of momentum, but it is still not clear how to quantify the benefits of NAG in FL in the literature. This motives us to implement NAG in FL. In FedNAG, each worker computes its own momentum individually in each local iteration. 
The worker momenta will be also aggregated by the aggregator and re-distributed to workers. It brings substantial challenges in convergence analysis as well as better performance. This is a key issue to be addressed by this paper. A detailed comparison of different momentum-based FL algorithms is shown in Table~\ref{tab:compare}. We also list important  notations in Table~\ref{tab:notation}.

%% file: 3modelDesign.tex
\subsection{Overview}
In the context of federated  learning, there are $N$
workers, located at different sites and communicating with an
 aggregator to learn a model $\textbf{w}^*$ which is a solution to
the following problem
\begin{align}
\min_{\mathbf{w}\in\mathbb{R}^d} F(\mathbf{w}) \triangleq \frac{\sum_{i=1}^{N}D_{i} F_{i}(\mathbf{w})}{D},
\end{align}
where $D_i$ is the number of data samples in worker $i$; $D=\sum_{i=1}^{N}D_i$ is the total number of data samples; and $d$ is the dimension of $\mathbf{w}$. $F_i(\cdot)$ is the local loss function at worker $i$ and $F(\cdot)$ is the global loss function. We assume $F_i(\cdot)$ satisfies the following conditions. 

\begin{enumerate}
\item $F_{i}(\mathbf{w})$ is convex.
\item $F_{i}(\mathbf{w})$ is $\rho$-Lipschitz, i.e., $\Vert F_{i}(\mathbf{w}_{1})-F_{i}(\mathbf{w}_{2})\Vert \leq \rho\Vert\mathbf{w}_{1}-\mathbf{w}_{2}\Vert$
for any $\mathbf{w}_{1}, \mathbf{w}_{2}$.
\item $F_{i}(\mathbf{w})$ is $\beta$-smooth, i.e., $\Vert\nabla F_{i}(\mathbf{w}_{1})-\nabla F_{i}(\mathbf{w}_{2})\Vert\leq\beta\Vert \mathbf{w}_{1}-$
$\mathbf{w}_{2} \Vert$ for any $\mathbf{w}_{1}, \mathbf{w}_{2}$.
\end{enumerate}
The above assumptions are widely adopted in a range of literature~\cite{wang2019adaptive, dinh2020federated, liu2020accelerating, liu2020client}. 

We assume all workers participate in the training.  This assumption matches with the setting of siloed data~\cite{kairouz2021advances}: Clients are different organizations, (e.g. medical or financial) geo-distributed datacenters. All clients are called in each global round.

\subsection{Algorithm}

Algorithm \ref{alg:FedNAG} demonstrates the implementation of FedNAG.  
We use $\mathbf{w}_i(t)$ and $\mathbf{v}_i(t)$ to denote the model parameter and momentum parameter in worker $i$ at $t$th iteration. Initially, at $t=0$, we set $\mathbf{v}_i(0) = \mathbf{0}$ and a same $\mathbf{w}_i(0)$ for all $i$. Each $\tau$ iterations will lead to  a global aggregation. 


Each iteration includes a local update, followed by a  global aggregation if $t=k\tau, k=1,2,\ldots$. 

\begin{algorithm}[tb]
\caption{FedNAG}
\label{alg:FedNAG}
\textbf{Input}: $\tau$, $T=K\tau$\\
\textbf{Output}: Final model parameter  $\mathbf{w}^{\mathrm{f}}$
\begin{algorithmic}[1] 
\STATE Initialize: $\mathbf{v}_i(0)=\mathbf{0}$, and $\mathbf{w}_i(0)$ as same value for all $i$.
\FOR{$t=1,2,\ldots,T$}
\STATE For each worker $i$ in parallel, compute its local update as \eqref{eq:localVi} and \eqref{eq:localWi}.
\IF {$t==k\tau$ where $k$ is a positive integer}
\STATE Aggregate $\mathbf{v}(t)$ and $\mathbf{w}(t)$ as \eqref{eq:globalV} and \eqref{eq:globalW}.
\STATE Set $\mathbf{v}_i(t) \gets \mathbf{v}(t)$ and $\mathbf{w}_i(t) \gets \mathbf{w}(t)$ for all $i$.
\ENDIF
\ENDFOR
\STATE \textbf{Set} $\mathbf{w}^{\mathbf{f}}$ as (\ref{eq:wf})
\end{algorithmic}
\end{algorithm}


\subsubsection{Local Updates}
In each iteration, the following update is conducted in each worker $i$,
\begin{align}
        \mathbf{v}_i(t) &\gets \gamma\mathbf{v}_i(t-1) - \eta\nabla F_i(\mathbf{w}_i(t-1)),\label{eq:localVi}\\
    \mathbf{w}_i(t) 
    &\gets \mathbf{w}_i(t-1) - \gamma\mathbf{v}_i(t-1) + (1+\gamma)\mathbf{v}_i(t) \label{eq:localWi}\nonumber\\
    &= \mathbf{w}_{i}(t-1) + \gamma\mathbf{v}_{i}(t) - \eta\nabla F_i(\mathbf{w}_{i}(t-1)),
\end{align}
where $\mathbf{v}_i(t)$ is the local momentum term at iteration $t$ in worker $i$ and  $\mathbf{w}_i(t)$ is the local model parameter at iteration $t$ in worker $i$. 
The above updates follow~\cite{bengio2013advances, nagCS231n}.  

\subsubsection{Global Aggregation}
If $t = k\tau, k=1,2,\ldots$, all workers will send $\mathbf{v}_i(t)$ and $\mathbf{w}_i(t)$ values to the aggregator and the aggregator calculates $\mathbf{v}(t)$ and $\mathbf{w}(t)$ as follows:
\begin{align}\label{eq:globalV}
\mathbf{v}(t) &\gets \frac{\sum_{i=1}^N D_i \mathbf{v}_i(t)}{D},\\
\label{eq:globalW}
\mathbf{w}(t) &\gets \frac{\sum_{i=1}^N D_i \mathbf{w}_i(t)}{D}.
\end{align}

Then aggregator will send back $\mathbf{v}(t)$ and $\mathbf{w}(t)$ to each worker $i$ to update $\mathbf{v}_i(t)\gets\mathbf{v}(t)$ and $\mathbf{w}_i(t)\gets\mathbf{w}(t)$.

Note that only if $t = k\tau$,  $\mathbf{v}(t)$ and $\mathbf{w}(t)$ are aggregated in \eqref{eq:globalV} and \eqref{eq:globalW}. For the purpose of analysis, we define $\mathbf{v}(t) = \frac{\sum_{i=1}^N D_i \mathbf{v}_i(t)}{D}$
 and $\mathbf{w}(t) = \frac{\sum_{i=1}^N D_i \mathbf{w}_i(t)}{D}$ at any iteration $t$ so that  $\mathbf{v}(t)$ and  $\mathbf{w}(t)$ can be used for convergence analysis.

After $T=K\tau$ iterations, the output $\mathbf{w}^{\mathrm{f}} $ is computed as follows: 
\begin{align} \label{eq:wf}
\mathbf{w}^{\mathrm{f}} \triangleq \underset{\mathbf{w} \in\{\mathbf{w}(k \tau): k=1,2, \ldots, K\}}{\arg \min } F(\mathbf{w}).
\end{align}

\subsection{Preliminary Analysis }

We present some simple preliminary analyses, which will be used in the rest of the paper. 

\subsubsection{Property of \texorpdfstring{$F(\mathbf{w})$}{F(w)}}
First, according to the assumptions, it is straightforward to show that $F(\mathbf{w})$ is convex, $\rho$-Lipschitz and $\beta$-smooth by applying triangle inequalities.

\subsubsection{Divergence of Gradient}
The divergence of gradient, which is commonly adopted in convergence analysis~\cite{wang2019adaptive, liu2020accelerating, liu2020client} can be defined as follows. 
\begin{definition}
(Gradient Divergence) For $\forall i$ and $\forall \mathbf{w}$, we define $\delta_i$ as the upper bound between $\nabla F_i(\mathbf{w})$ and $\nabla F(\mathbf{w})$, i.e.,
\begin{align}\label{definition:GD}
    \Vert \nabla F_i(\mathbf{w}) - \nabla F(\mathbf{w}) \Vert \leq \delta_i.
\end{align}
We also define
\begin{align} \label{eq:delta=}
    \delta \triangleq \frac{\sum_i D_i\delta_i}{D}.
\end{align}
\end{definition}
Please note that $\delta_i$ is different at different workers, indicating the datasets at different workers may not be independent and identically distributed (non-i.i.d.) \cite{wang2019adaptive}. 

\subsubsection{Virtual Updates}
We use $[k]$ to denote interval $t\in[(k-1)\tau, k\tau]$ for $k=1,2,3,\ldots,K$. It shows $\tau$ iterations within two global aggregations.

In each interval $[k]$, first, at $(k-1)\tau$, we set
\begin{align}
    \mathbf{v}_{[k]}((k-1)\tau) &\gets \mathbf{v}((k-1)\tau)\label{eq:vk(k-1)tau},\\
    \mathbf{w}_{[k]}((k-1)\tau) &\gets \mathbf{w}((k-1)\tau)\label{eq:wk(k-1)tau}.
\end{align}
$\mathbf{v}_{[k]}((k-1)\tau)$ and $\mathbf{w}_{[k]}((k-1)\tau)$ are set as the aggregated values right after the global aggregation is conducted.


Second, starting from the aggregated values,  we consider virtual updates as if centralized NAG is adopted. In iterations $(k-1)\tau<t\leq k\tau$, we conduct 
\begin{align} 
        \mathbf{v}_{[k]}(t) &\gets \gamma\mathbf{v}_{[k]}(t-1) - \eta\nabla F(\mathbf{w}_{[k]}(t-1)),\label{eq:vkt=}\\
\mathbf{w}_{[k]}(t) 
&\gets \mathbf{w}_{[k]}(t-1) - \gamma\mathbf{v}_{[k]}(t-1) + (1+\gamma)\mathbf{v}_{[k]}(t) \label{eq:wkt=}\nonumber\\
&= \mathbf{w}_{[k]}(t-1) + \gamma\mathbf{v}_{[k]}(t) - \eta\nabla F(\mathbf{w}_{[k]}(t-1)).
\end{align}
We repeat the above process for each $[k]$. 
These $\mathbf{w}_{[k]}(t)$ and $\mathbf{v}_{[k]}(t)$ are virtual values assuming there is a centralized update. They are used to bound the gap to prove the convergence shortly. Please note that $\mathbf{w}_{[k]}(k\tau)$ and 
$\mathbf{w}_{[k+1]}(k\tau)$ are different. $\mathbf{w}_{[k]}(k\tau)$ is calculated from $\mathbf{w}_{[k]}((k-1)\tau)$ after $\tau$ iterations of centralized update, and $\mathbf{w}_{[k+1]}(k\tau)$ is directly given by $\mathbf{w}(k\tau)$. Fig.~\ref{fig:wt} illustrates the evolution of $\mathbf{w}_i(t)$, $\mathbf{w}(t)$, and $\mathbf{w}_{[k]}(t)$. 

\begin{figure}
    \centering
    \includegraphics[width=0.48\textwidth]{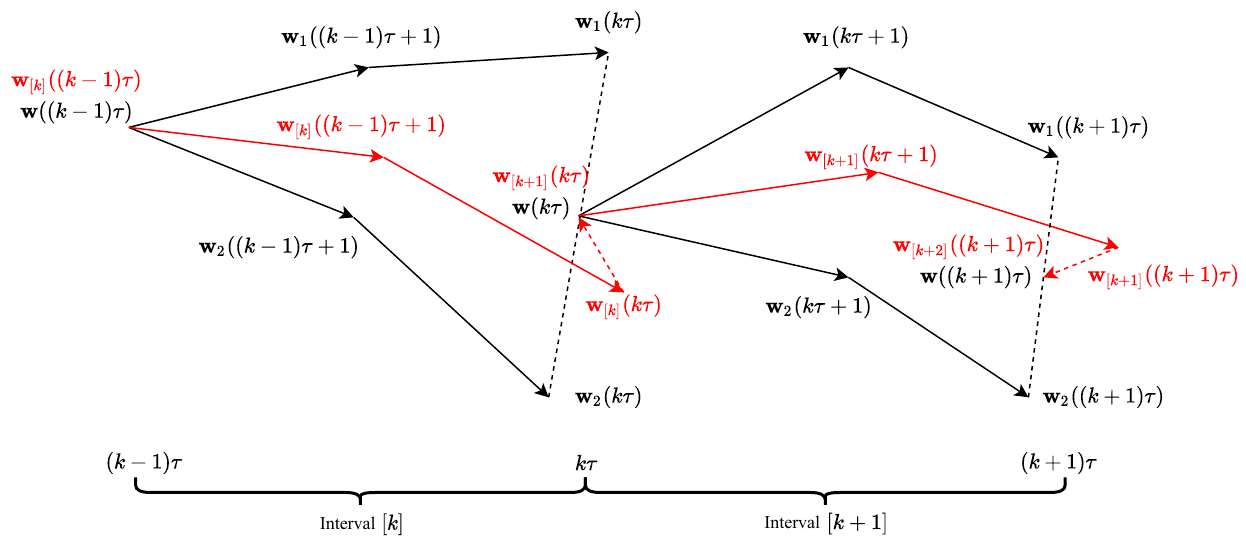}
    \caption{Illustration of $\mathbf{w}(t)$, when $N=2, \tau=2$}
    \label{fig:wt}
\end{figure}

%% file: 4convergenceAnalysis2.tex
In this section, we provide detailed convergence analysis of FedNAG. This includes two steps: We first bound the gap of the weight $\mathbf{w}$ between FedNAG and virtual updates; Then we bound the loss function $F(\mathbf{w})$ between FedNAG and the optimal solution. 

\subsection{Bounding \texorpdfstring{$\Vert\mathbf{w}(t)-\mathbf{w}_{[k]}(t) \Vert $}{||w(t)-w[k](t||}}
We firstly analyze the upper bound between $\mathbf{w}(t)$ and $\mathbf{w}_{[k]}(t)$, leading to the following theorem.

\begin{theorem}\label{theorem:wt-wkt}
For any interval $[k]$, $\forall t \in [k]$, we have:
\begin{align}\label{ieq:wt-wkt}
    \Vert\mathbf{w}(t)-\mathbf{w}_{[k]}(t) \Vert \leq h(t-(k-1)\tau),
\end{align}
where we define
\begin{align*}
&A\triangleq\frac{(1+\eta\beta)(1+\gamma)+\sqrt{(1+\eta\beta)^2(1+\gamma)^{2}-4\gamma(1+\eta\beta)}}{2\gamma},\\
&B\triangleq\frac{(1+\eta\beta)(1+\gamma)-\sqrt{(1+\eta\beta)^2(1+\gamma)^{2}-4\gamma(1+\eta\beta)}}{2\gamma},\\
&E\triangleq\frac{\gamma A+A-1}{(A-B)(\gamma A-1)},\\
&F\triangleq\frac{\gamma B+B-1}{(A-B)(1-\gamma B)},
\end{align*}
and $h(x)$ yields
\begin{flalign} \label{eq:h(x)}
h(x)=& \eta \delta\left[E(\gamma A)^{x}+F(\gamma B)^{x}-\frac{1}{\eta \beta}\right.\nonumber\\
&\left.-\frac{\gamma^2(\gamma^{x}-1)-(\gamma-1) x}{(\gamma-1)^{2}}\right]
\end{flalign}
for $0<\gamma<1$ and any $x = 0,1,2,\dots$

We note that $F(\mathbf{w})$ is $\rho$-Lipschitz, so we also have:
\begin{align}\label{ieq:Fwt-Fwkt}
F(\mathbf{w}(t))-F(\mathbf{w}_{[k]}(t))\leq \rho h(t-(k-1)\tau).
\end{align}
\end{theorem}
\begin{proof}
See Appendix for detailed proof.
\end{proof}

We have the following observations on Theorem \ref{theorem:wt-wkt}.

\noindent \textcircled{1} \textbf{Monotone of $h(x)$.} $h(0)= h(1)= 0$ and $h(x)$ increases with respect to integer $x$ for $x \geq 1$. See Appendix for detailed proof.

\noindent  \textcircled{2} \textbf{Property of $h(0)$.} When $x=0$, we have  $t=(k-1)\tau$ (the beginning of interval $[k]$) and the upper bound in \eqref{ieq:wt-wkt} is $0$. This is consistent with \eqref{eq:vk(k-1)tau} and \eqref{eq:wk(k-1)tau} for any $k$.

\noindent \textcircled{3} \textbf{Property of $h(1)$.} When $x=1$, we have $t=(k-1)\tau+1$ (the beginning of second iteration of interval $[k]$) and the upper bound in \eqref{ieq:wt-wkt} is still zero. It is easy to verify that if all workers conduct global aggregation right after the end of the first local iteration, there is no gap between FedNAG and centralized NAG.

\noindent \textcircled{4} \textbf{Property of $\tau=1$.} When $\tau =1$, we have $t-(k-1)\tau=0$ or $1$. Thus, for any interval $k$ and $t\in[k]$, the gap in \eqref{ieq:wt-wkt} and \eqref{ieq:Fwt-Fwkt} is always zero. This means that FedNAG is equivalent to centralized NAG when there is only one local update step between two global aggregation steps. See Appendix for detailed discussion. 

\noindent \textcircled{5}  \textbf{Property of $\tau>1$.} When $\tau >1$, because $t\in [(k-1)\tau,k\tau]$, we have $x=t-(k-1)\tau\in[0,\tau] $. Thus, the value of $x$ could be larger when $\tau$ is large. According to the definitions of $A,B,E$, and $F$, we can see that $\gamma A >1, 0<\gamma B<1, E>0, F>0$. When $x$ is large, because $0<\gamma<1$, the last term in \eqref{eq:h(x)} will linearly decrease with respect to $x$. Therefore, for \eqref{eq:h(x)}, $E(\gamma A)^x$  dominates when $x$ is large. It means the upper bound in \eqref{ieq:wt-wkt} will be exponentially increased with $t\in [k]$.

\noindent \textcircled{6} \textbf{Impact of $\delta$.} $h(x)$ increases linearly with respect to $\delta$.  The value of $\delta$ reflects the difference of data distribution in each worker. Larger divergence of data distribution leads to larger gap of $h(x)$.  

\subsection{Bounding \texorpdfstring{$F(\mathbf{w}(T))-F(\mathbf{w}^{*})$}{F(w(T))-F(w*)}}
\label{sec: FwT-Fw*}


For convenience, we define
\begin{align*}
p& \triangleq \max _{k\in [1,K], t \in[k]} \frac{\left\|\gamma\mathbf{v}_{[k]}(t)\right\|}{\left\|\eta\nabla F\left(\mathbf{w}_{[k]}(t)\right)\right\|},\\
\omega& \triangleq \min _{k\in [1,K], t \in[k]} \frac{1}{\left\|\mathbf{w}_{[k]}(t)-\mathbf{w}^{*}\right\|^{2}}.
\end{align*}
We can obtain the following theorem to get the upper bound as follows.
\begin{theorem} \label{theorem:Fwt-Fw*}
When all the following conditions are satisfied:
\begin{enumerate}
    \item $0<\beta\eta(\gamma+1)\leq 1$ and $0\leq\gamma<1$,
    \item $\omega\alpha-\frac{\rho h(\tau)}{\tau \varepsilon^{2}}>0$,
    \item $F(\mathbf{w}_{[k]}(k \tau))-F\left(\mathbf{w}^{*}\right) \geq \varepsilon$ for all $k$,
    \item $F(\mathbf{w}(T))-F(\mathbf{w}^{*}) \geq \varepsilon$,
\end{enumerate}
for some $\varepsilon >0$, the convergence upper bound of Algorithm 1 after $T$ iterations is given by
\begin{equation} \label{ieq:FwT-Fw*}
F(\mathbf{w}(T))-F(\mathbf{w}^{*}) \leq \frac{1}{T\left(\omega \alpha-\frac{\rho h(\tau)}{\tau \varepsilon^{2}}\right)},
\end{equation}
where we define
\begin{align*}
\alpha \triangleq &\eta(\gamma+1)\left(1-\frac{\beta\eta(\gamma+1)}{2}\right)-\frac{\beta\eta^2\gamma^2 p^2}{2}\\
&-\frac{(1-\beta\eta(\gamma+1))(1+\eta^2\gamma^2 p^2)}{2}.
\end{align*}
\end{theorem}
\begin{proof}
See Appendix for detailed proof.
\end{proof}
Through Theorem~\ref{theorem:Fwt-Fw*}, we can further obtain the following bound between $F\left(\mathbf{w}^{\mathrm{f}}\right)$ and $F\left(\mathbf{w}^{*}\right)$.
\begin{theorem} \label{theorem:Fwf-Fw*}
When $0<\beta\eta(\gamma+1)\leq 1$, and $0\leq\gamma< 1$, we have 
\begin{align} \label{ieq:Fwf-Fw*}
&F(\mathbf{w}^{\mathrm{f}})-F\left(\mathbf{w}^{*}\right)\nonumber\\
\leq& \frac{1}{2 T \omega \alpha}+\sqrt{\frac{1}{4 T^{2} \omega^{2} \alpha^{2}}+\frac{\rho h(\tau)}{\omega \alpha \tau}}+\rho h(\tau).
\end{align}
\end{theorem}
\begin{proof}
See Appendix for detailed proof.
\end{proof}

We have proven that FedNAG has the convergence rate $\mathcal{O}\left(\frac{1}{T}\right)$ for convex problems. Please note we have the following observations on Theorem \ref{theorem:Fwf-Fw*}.

\noindent \textcircled{1} \textbf{Effect of $\tau$.}
From Appendix, we have known that $h(\tau)\geq0$ and increases with integer $\tau$. Thus, for a given $T$, the convergence upper bound becomes larger when $\tau$ is larger.

\noindent \textcircled{2} \textbf{Property of $\tau=1$.} When $\tau =1$, we have $h(\tau)=0$. We can observe that the gap converges to zero when $T\to\infty$. This means if we conduct global aggregation after every local update, $F(\mathbf{w}(t))$ will converge to the optimal solution. 

\noindent \textcircled{3} \textbf{Property of $\tau>1$.} When $\tau>1$, we have $h(\tau)>0$. We can observe that the gap converges to a non-zero gap $\sqrt{\frac{\rho h(\tau)}{\omega\alpha\tau}}+\rho h(\tau)$ when $T\to \infty$. This means if we conduct global aggregation after multiple local updates, there is a non-zero gap to the optimal solution.

\noindent \textcircled{4} \textbf{Tradeoff between communication and convergence.} Based on the Observations \textcircled{2} and \textcircled{3} above, $\tau =1$ gives the best convergence performance.  However, by doing so, it will increase the communication frequency.   This will lead to a tradeoff between communication overhead and convergence performance. In this paper, we do not model the costs and utilities of communication overhead (in different types of distributed systems) and convergence performance, so that the optimal tradeoff is left for future work.

\noindent \textcircled{5} \textbf{Effect of $\delta$.} 
Following the Observation \textcircled{6} of Theorem \ref{theorem:wt-wkt},  the convergence upper bound will be increased when $\delta$ is getting larger.

\section{Comparison between FedAvg and FedNAG}
In this section, we compare the performance between FedAvg and FedNAG. The  convergence upper bound of FedAvg has been derived in Theorem 2 in \cite{wang2019adaptive} as follows:
\begin{align}
&F\left(\hat{\mathbf{w}}^{\mathrm{f}}\right)-F\left(\mathbf{w}^{*}\right)\nonumber\\
\leq& \frac{1}{2T\omega\hat{\alpha}}+\sqrt{\frac{1}{4T^2 \omega^{2} \hat{\alpha}^{2}}+\frac{\rho \hat{h}(\tau)}{\omega \hat{\alpha} \tau}}+\rho \hat{h}(\tau),
\end{align}
where
\begin{align}\label{eq:hat_h(x)}
\hat{h}(\tau)&=\frac{\delta}{\beta}\left((\eta \beta+1)^{\tau}-1\right)-\eta\delta\tau,\\
\hat{\alpha}&\triangleq\eta\left(1-\frac{\beta\eta}{2}\right).\nonumber
\end{align}
Please note that $\rho, \beta, \tau, \omega$, and $\eta$ are defined the same way as those in FedNAG in this paper. $\hat{\alpha}$ and $\hat{h}(\cdot)$ are defined differently, but with similar meanings as $\alpha$ and $h(\cdot)$ in this paper.


Although FedNAG has the same convergence rate as FedAvg, we can still compare the convergence performance by comparing the convergence upper bound for a given $T$. In order to make a fair comparison, we let FedAvg and FedNAG trained under the same environment using the same configuration. Here, we note that $\delta$ and $\omega$ reflect the properties of data distribution. We assume  the dataset is distributed in each worker in the same way in FedAvg and FedNAG, so that the values of $\omega$ and $\delta$ are same. The loss function $F_{i}(\cdot)$, $F(\cdot)$, constants $\rho$ and $\beta$, and hyper-parameters $\tau$ and $\eta$ are the same. We also set the same initial value for $\mathbf{w}^{\mathrm{f}}$, $\mathbf{w}_i(0)$ for FedAvg and FedNAG. The only new term in FedNAG is $\mathbf{v}_i(t)$, and we  set $\mathbf{v}_i(0)=\mathbf{0}$.

We use $f_1(T)$ and $f_2(T)$ to define the convergence upper bound of FedNAG and FedAvg respectively. Small function value implies better convergence performance. 
\begin{align}
f_1(T) &\triangleq \frac{1}{2 T \omega \alpha}+\sqrt{\frac{1}{4 T^{2} \omega^{2} \alpha^{2}}+\frac{\rho h(\tau)}{\omega \alpha \tau}}+\rho h(\tau),\\
f_2(T) &\triangleq \frac{1}{2T\omega\hat{\alpha}}+\sqrt{\frac{1}{4T^2 \omega^{2} \hat{\alpha}^{2}}+\frac{\rho \hat{h}(\tau)}{\omega \hat{\alpha} \tau}}+\rho \hat{h}(\tau).
\end{align}

To prevent the gradient descent from overshooting the minimum or failing to converge \cite{goodfellow2016deep}, we choose a sufficiently small $\eta$  to guarantee the convergence of FedNAG and FedAvg. The following conclusion is made when $\eta\to 0^+$.

\begin{theorem} \label{theorem:fast}
When $0<\beta\eta(\gamma+1)\leq 1$ and $0<\gamma< 1$, 
FedNAG outperforms FedAvg, i.e.,
\begin{align*}
    f_1(T) < f_2(T)
\end{align*}
for any $T$ and an arbitrarily small $\eta\to 0^+$.
\end{theorem}
\begin{proof}
See Appendix for detailed discussion.
\end{proof}
Please note that we have the following observations on Theorem \ref{theorem:fast}.

\noindent \textcircled{1} \textbf{Discussion of $\eta$.}
In Theorem \ref{theorem:fast}, we set $\eta\to0^+$. Actually, there exists a  threshold value for $\eta$ called $\bar{\eta}$. 
If $\eta<\bar{\eta}$, $0<\beta\eta(\gamma+1)\leq 1$, and $0<\gamma< 1$, then $f_1(T)< f_2(T)$ is still true. Numerical method can be used to calculate the value of $\bar{\eta}$.

%% file: 5experiment.tex
In this section we evaluate the convergence performance of FedNAG compared with benchmark algorithms including FedAvg, FedMom, centralized SGD (cSGD), and centralized NAG (cNAG) by real-world experiments. We then discuss the impacts of hyper-parameters, including global aggregation frequency $\tau$, momentum coefficient $\gamma$, and number of workers $N$. 
Then, we explicitly generate different levels of non-i.i.d. data to test the performance of FedNAG and benchmarks. Finally, we perform trace-driven simulation as a digital representation of the mobile edge computing environment to analyze the total training time (including computation delay and communication delay).

\begin{table}[t!b]
\centering
\caption{Experiment settings}
\label{tab:experiment_setting}
\begin{tabular}{ p{0.20\columnwidth} p{0.132\columnwidth} p{0.09\columnwidth} p{0.09\columnwidth} p{0.11\columnwidth} p{0.09\columnwidth}}
\toprule
Experiment\newline purpose & Figures & $\tau$ & $\gamma$ & $T$ & $N$\\
\midrule
Convergence\newline performance &Fig. \ref{fig:alg} & $20$ or \newline  $40$ & $0.9$ & $1000$ or \newline$10000$ & $4$\\  
\midrule
Effects of $\tau$ & Figs. \ref{fig:tau} and \ref{fig:tau_reach}& various & $0.5$ & $1000$  & $4$ \\
\midrule
Effects of $\gamma$&Figs. \ref{fig:gamma},\newline \ref{fig:gamma_T}, \ref{fig:gamma=1} & $4$ & various & $1000$  & $4$ \\
\midrule
Effects of $N$ &Fig. \ref{fig:worker} & $4$ & $0.9$ & $1000$  & various \\
\midrule
Effects of non-\newline i.i.d. data & Fig. \ref{fig:noniid} & $40$ & $0.9$ & $1000$  & $4$ \\
\midrule
Trace-driven \newline simulation & Fig. \ref{fig:trace_driven} & $20$ or \newline $40$ & $0.9$ & $1000$  & $4$ \\
\bottomrule
\end{tabular}
\end{table}

\begin{figure*}[htb!]
    \centering
    \subfigure[Linear regression on MNIST]{
    \begin{minipage}[t]{0.24\textwidth}
        \centering
        \includegraphics[width=\textwidth]{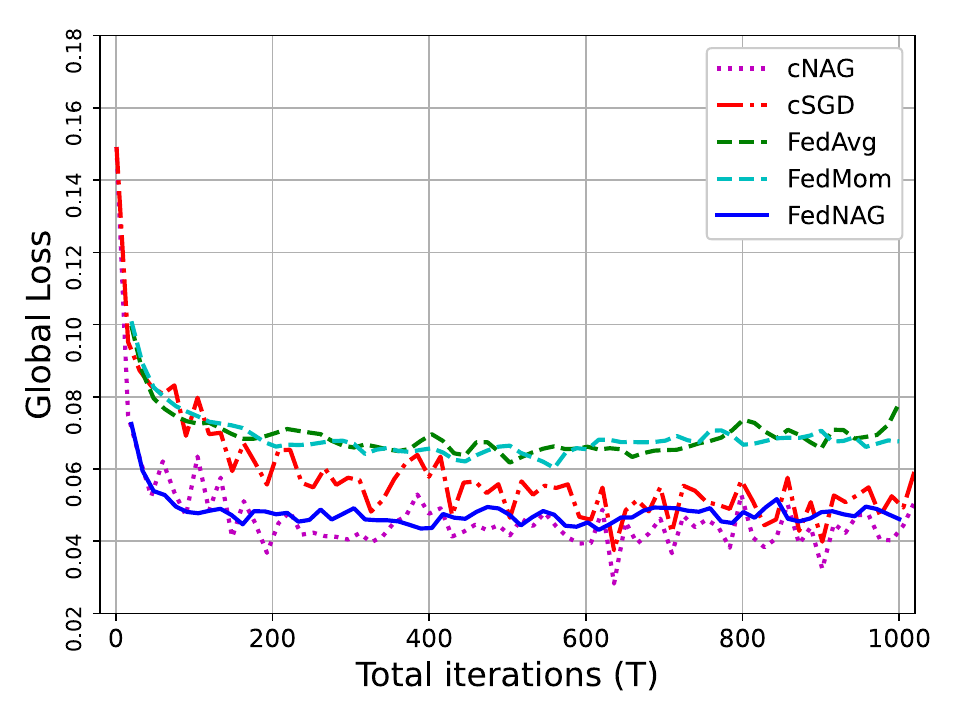}
    \end{minipage}
    \begin{minipage}[t]{0.24\textwidth}
        \centering
        \includegraphics[width=\textwidth]{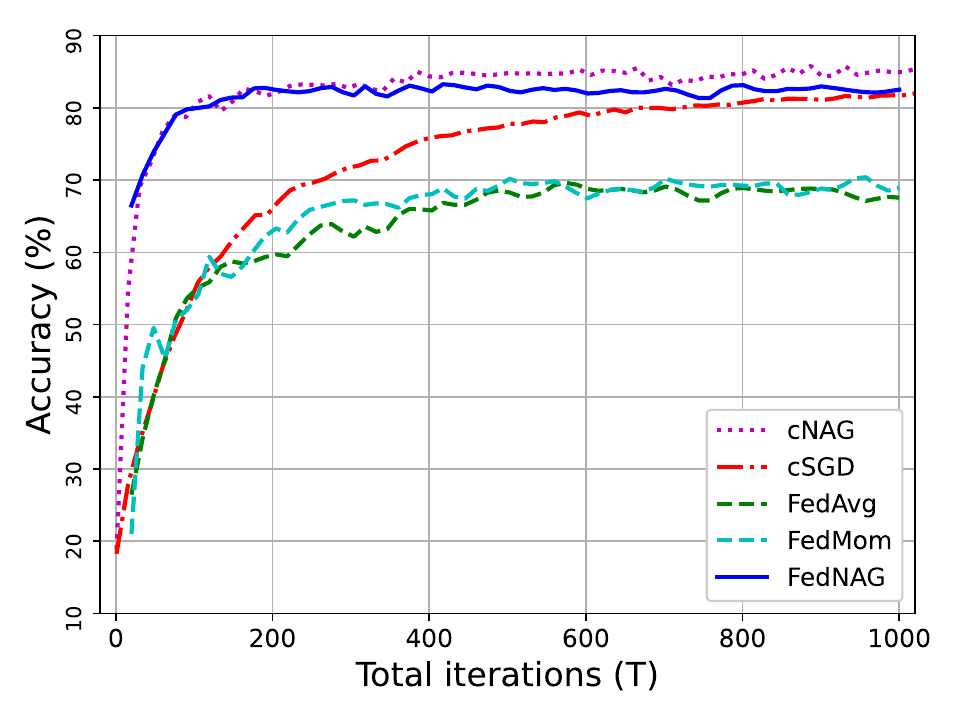}
    \end{minipage}\label{fig:linear}}
    \subfigure[Logistic regression on MNIST]{
    \begin{minipage}[t]{0.24\textwidth}
        \centering
        \includegraphics[width=\textwidth]{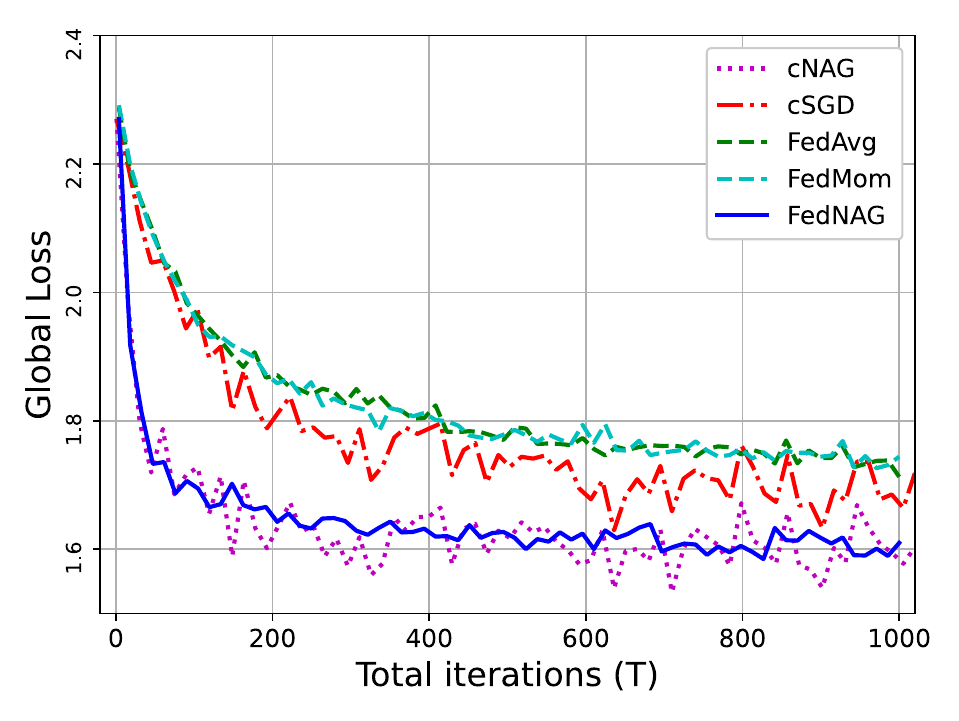}
    \end{minipage}
    \begin{minipage}[t]{0.24\textwidth}
        \centering
        \includegraphics[width=\textwidth]{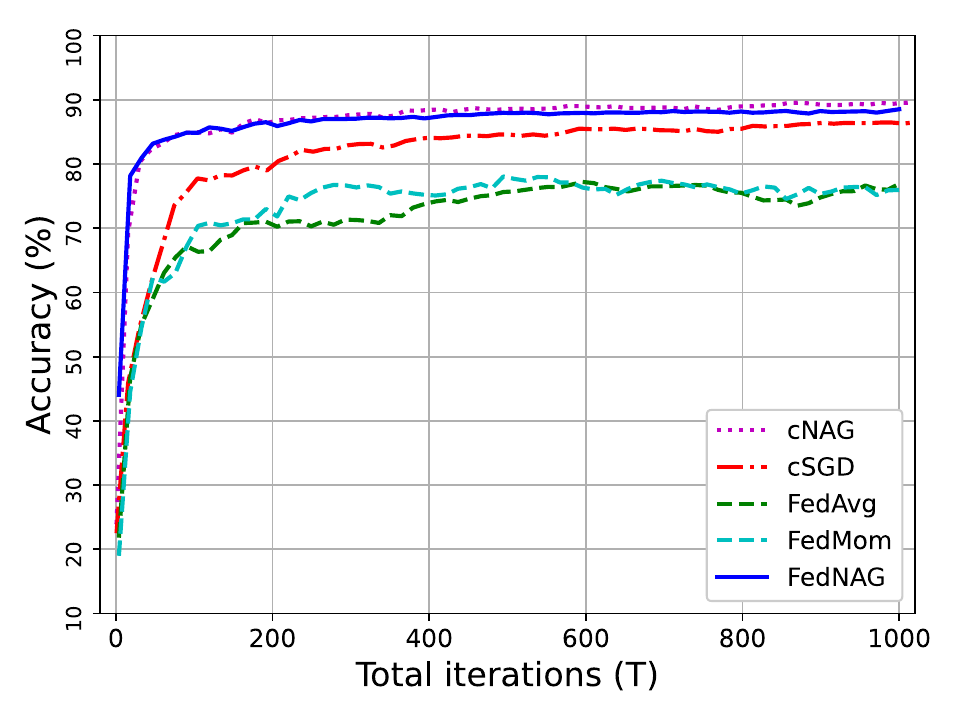}
    \end{minipage}\label{fig:logistic}}
    \subfigure[CNN on MNIST]{
    \begin{minipage}[t]{0.24\textwidth}
        \centering
        \includegraphics[width=\textwidth]{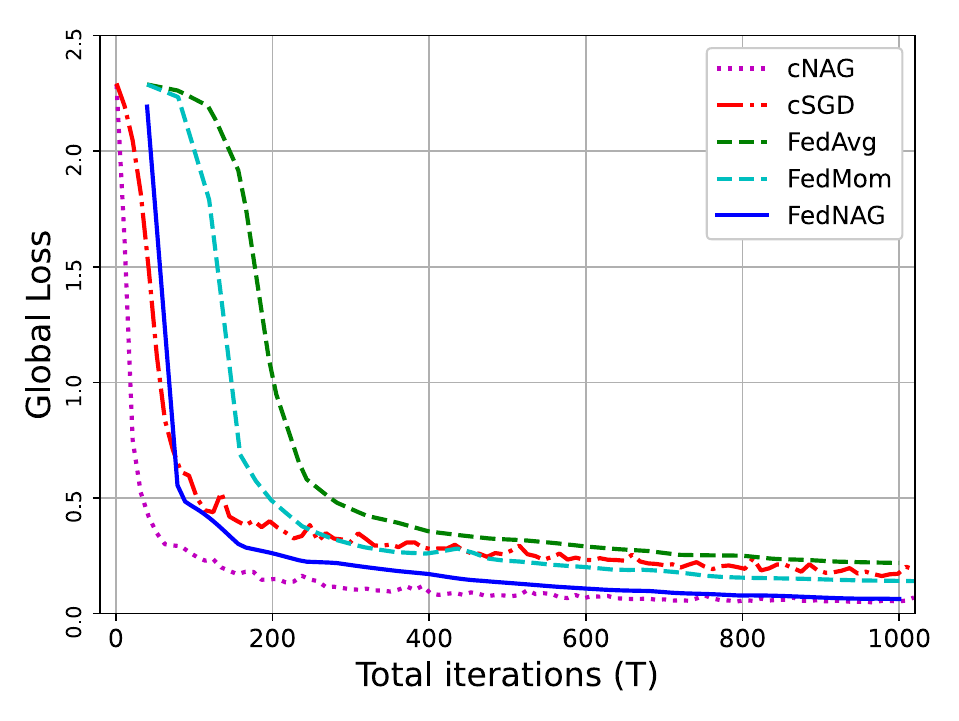}
    \end{minipage}
    \begin{minipage}[t]{0.24\textwidth}
        \centering
        \includegraphics[width=\textwidth]{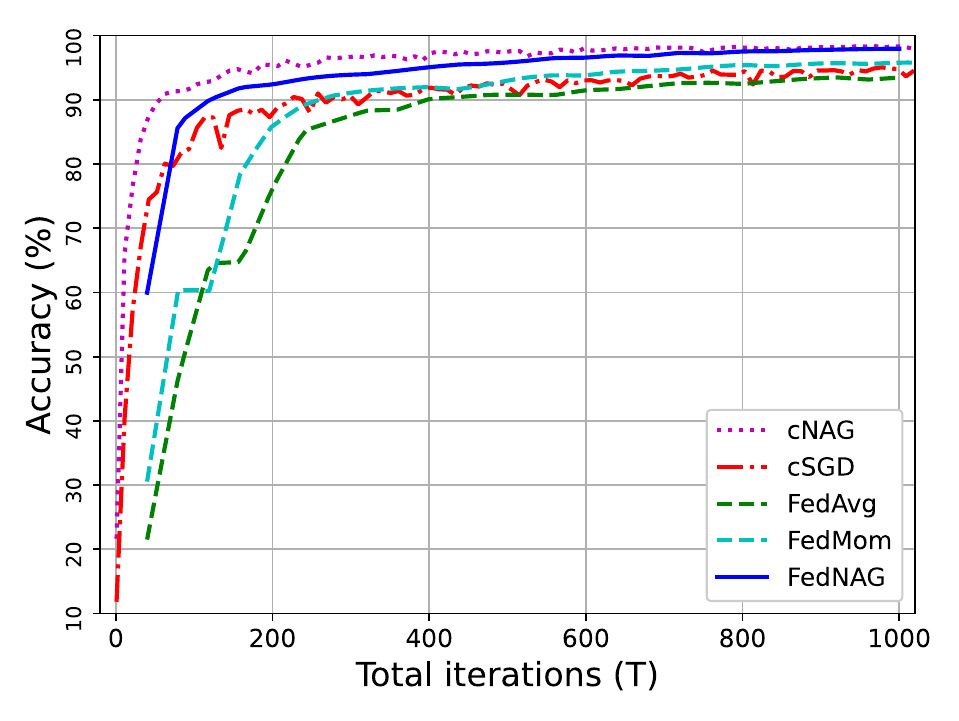}
    \end{minipage}\label{fig:cnn_mnist}}
    \subfigure[CNN on CIFAR-10]{
    \begin{minipage}[t]{0.24\textwidth}
        \centering
        \includegraphics[width=\textwidth]{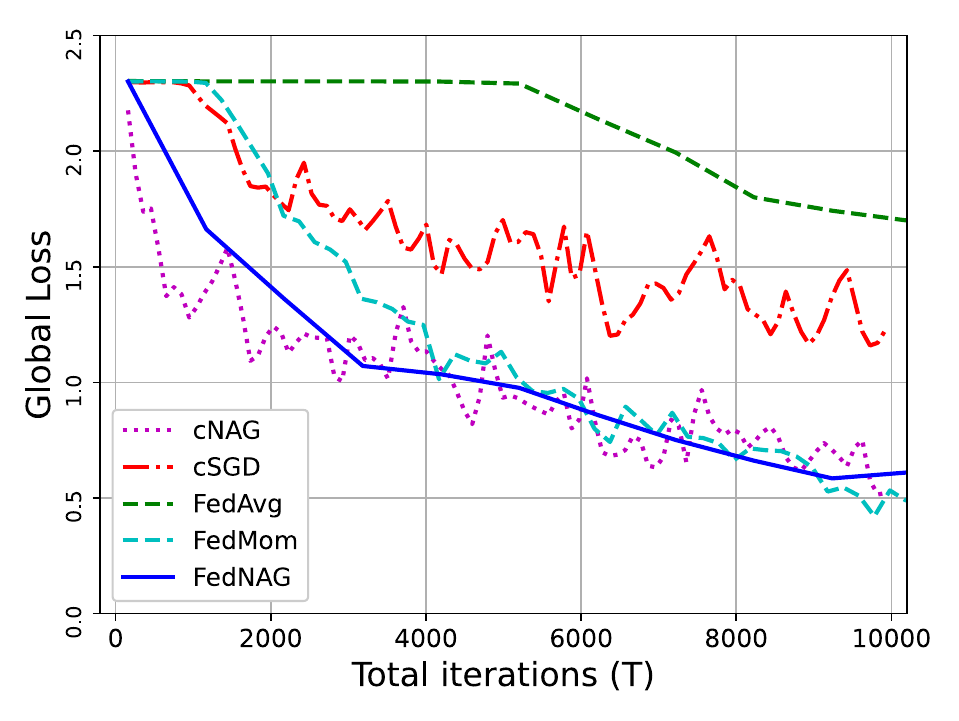}
    \end{minipage}
    \begin{minipage}[t]{0.24\textwidth}
        \centering
        \includegraphics[width=\textwidth]{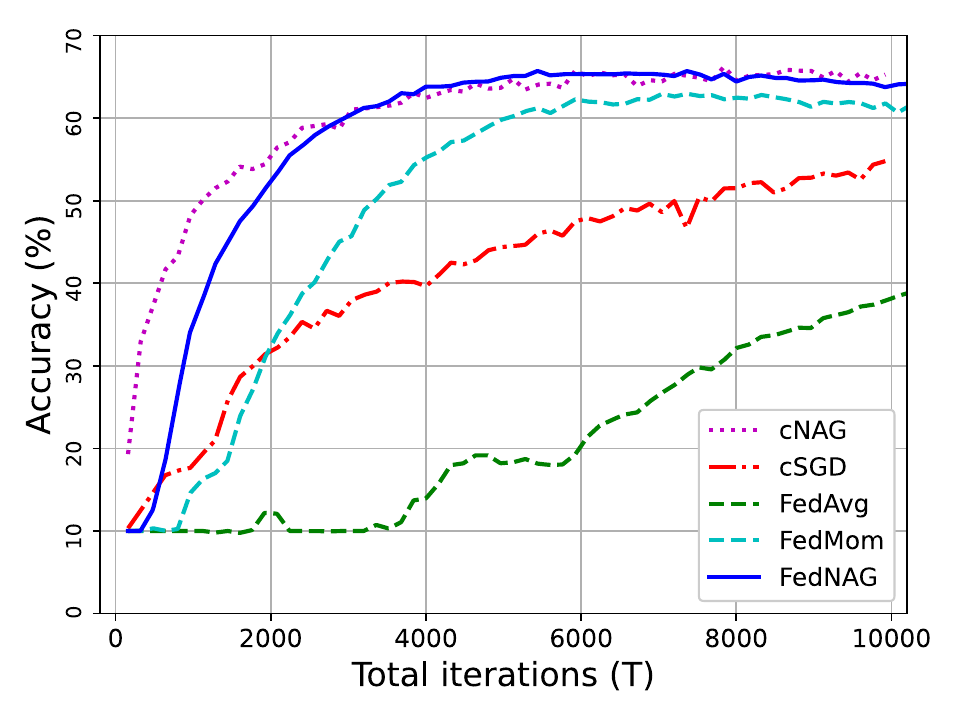}
    \end{minipage}\label{fig:cifar10}}
    \subfigure[VGG16 on CIFAR-10]{
    \begin{minipage}[t]{0.24\textwidth}
        \centering
        \includegraphics[width=\textwidth]{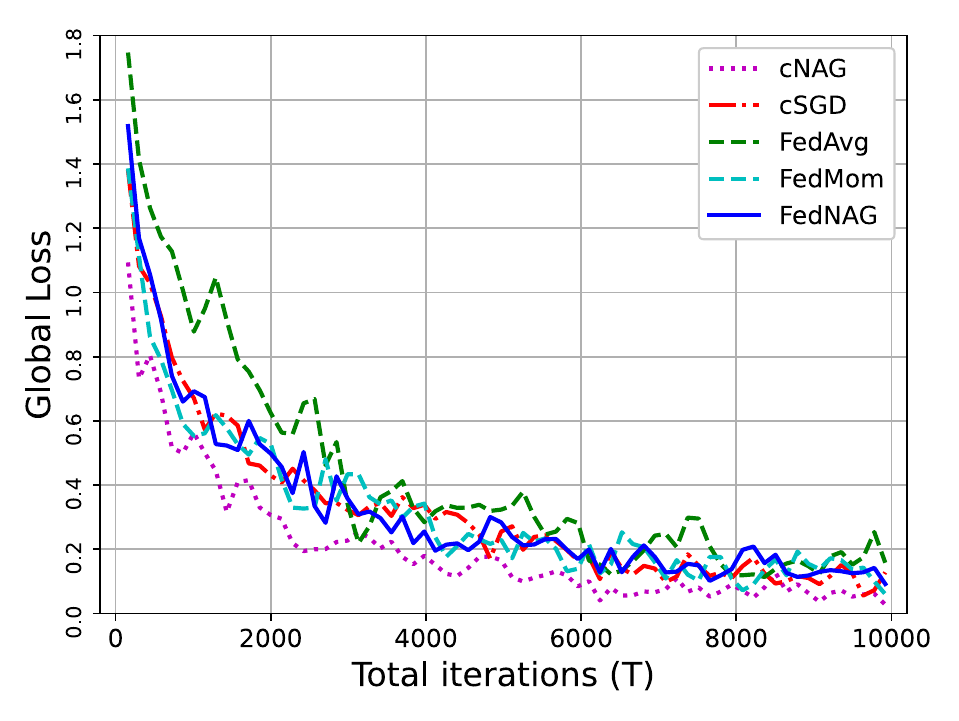}
    \end{minipage}
    \begin{minipage}[t]{0.24\textwidth}
        \centering
        \includegraphics[width=\textwidth]{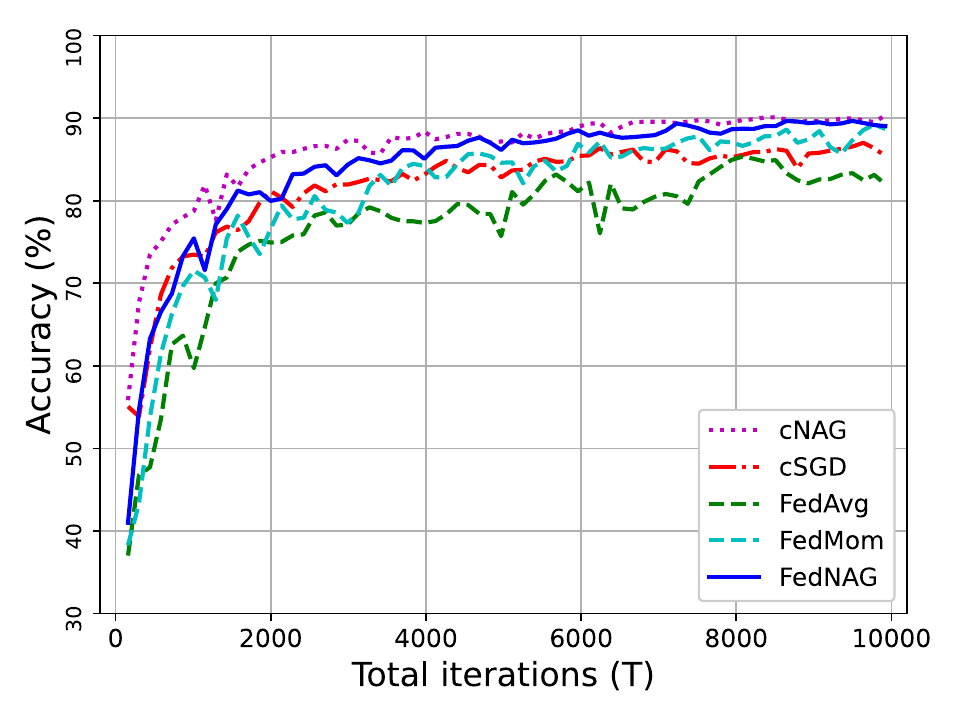}
    \end{minipage}\label{fig:vgg16_cifar10}}
    \subfigure[VGG16 on CIFAR-100]{
    \begin{minipage}[t]{0.24\textwidth}
        \centering
        \includegraphics[width=\textwidth]{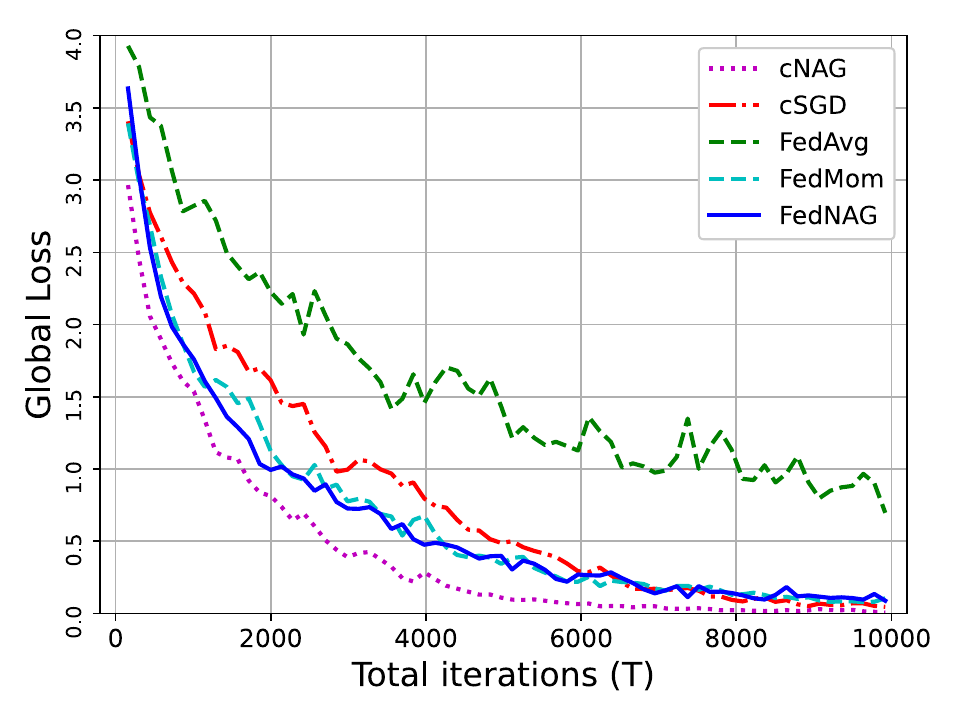}
    \end{minipage}
    \begin{minipage}[t]{0.24\textwidth}
        \centering
        \includegraphics[width=\textwidth]{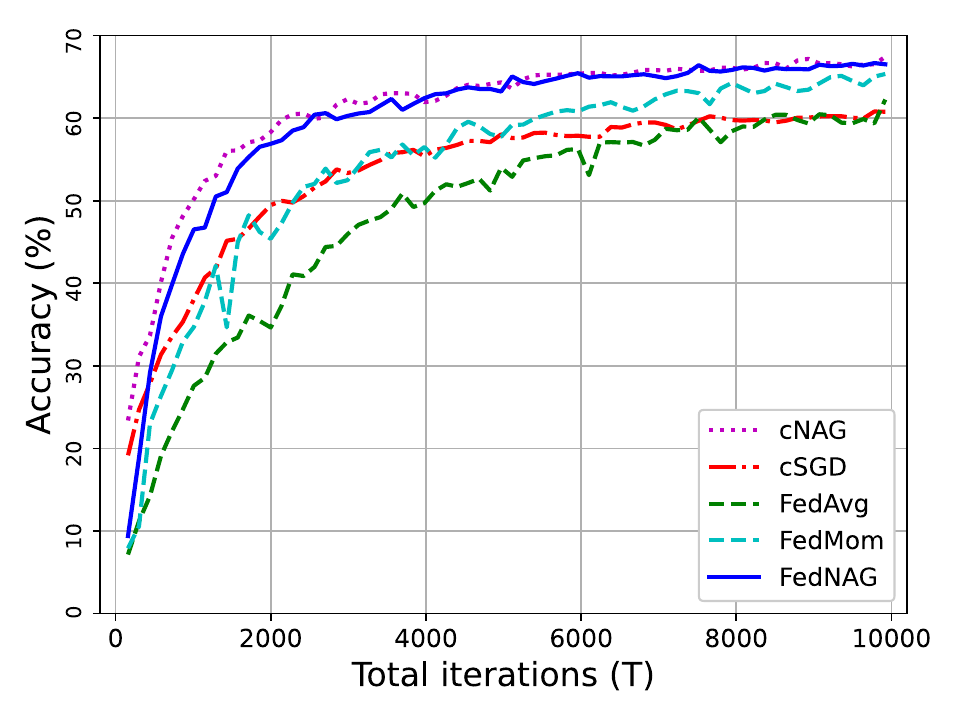}
    \end{minipage}\label{fig:vgg16_cifar100}}
    \caption{Convergence performance with benchmark algorithms}
    \label{fig:alg}
\end{figure*}

\begin{figure*}[htb!]
    \centering
    \subfigure[Effect of aggregation frequency $\tau$]{
    \begin{minipage}[t]{0.24\textwidth}
        \centering
        \includegraphics[width=\textwidth]{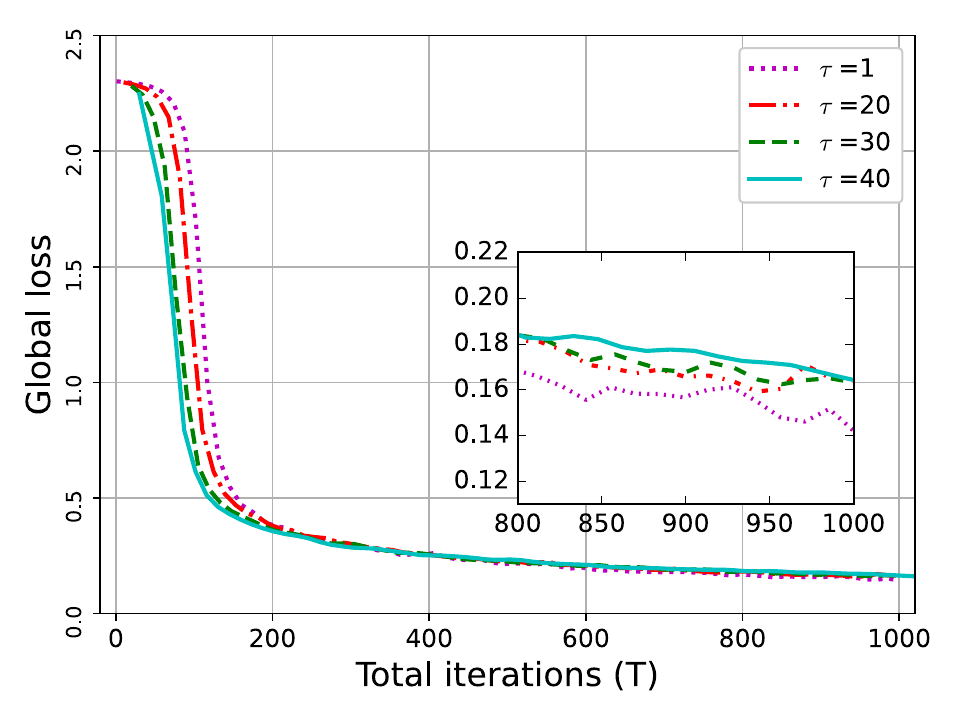}
    \end{minipage}
    \begin{minipage}[t]{0.24\textwidth}
        \centering
        \includegraphics[width=\textwidth]{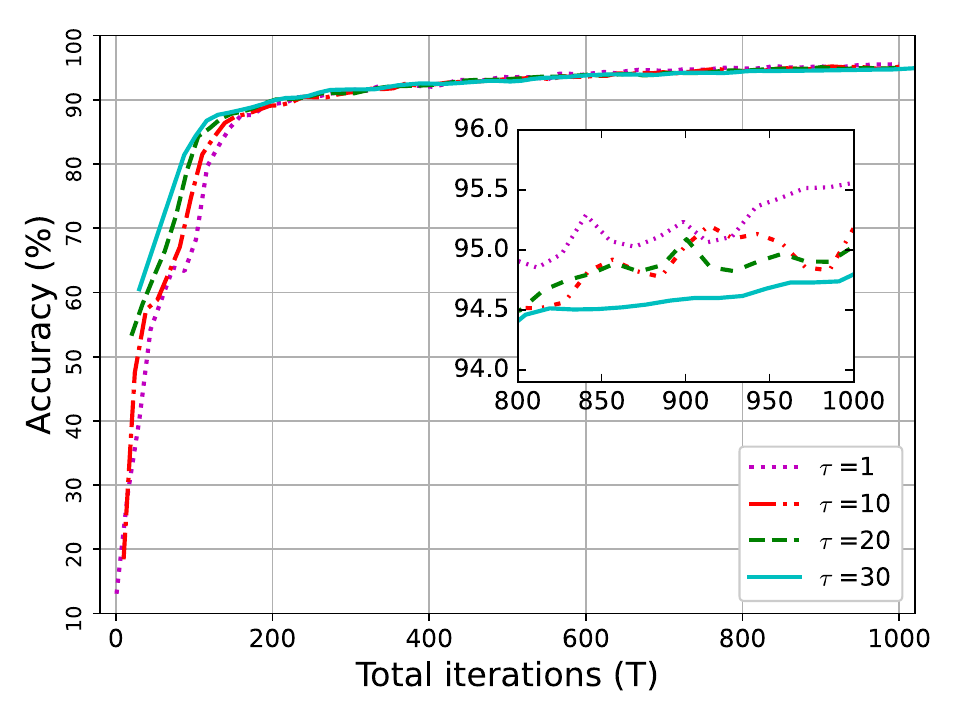}
    \end{minipage}\label{fig:tau}}
    \subfigure[Total iterations when global loss reaches $0.5$ (left) and accuracy reaches $85\%$ (right) for different $\tau$]{
    \begin{minipage}[t]{0.24\textwidth}
        \centering
        \includegraphics[width=\textwidth]{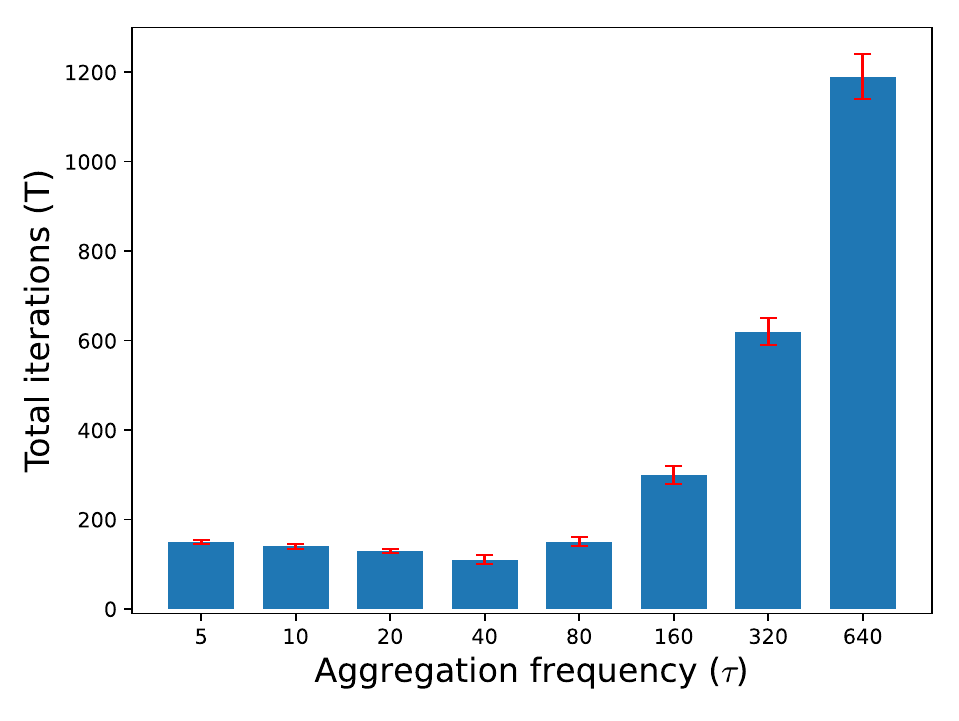}
    \end{minipage}
    \begin{minipage}[t]{0.24\textwidth}
        \centering
        \includegraphics[width=\textwidth]{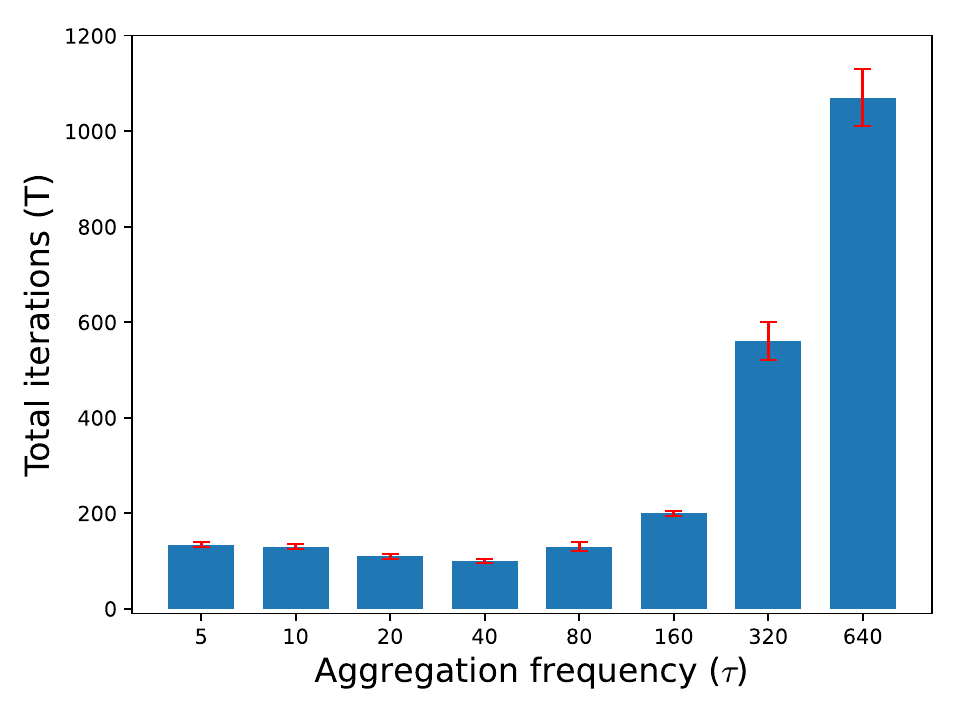}
    \end{minipage}\label{fig:tau_reach}}
    \subfigure[Effect of momentum coefficient $\gamma$]{
    \begin{minipage}[t]{0.24\textwidth}
        \centering
        \includegraphics[width=\textwidth]{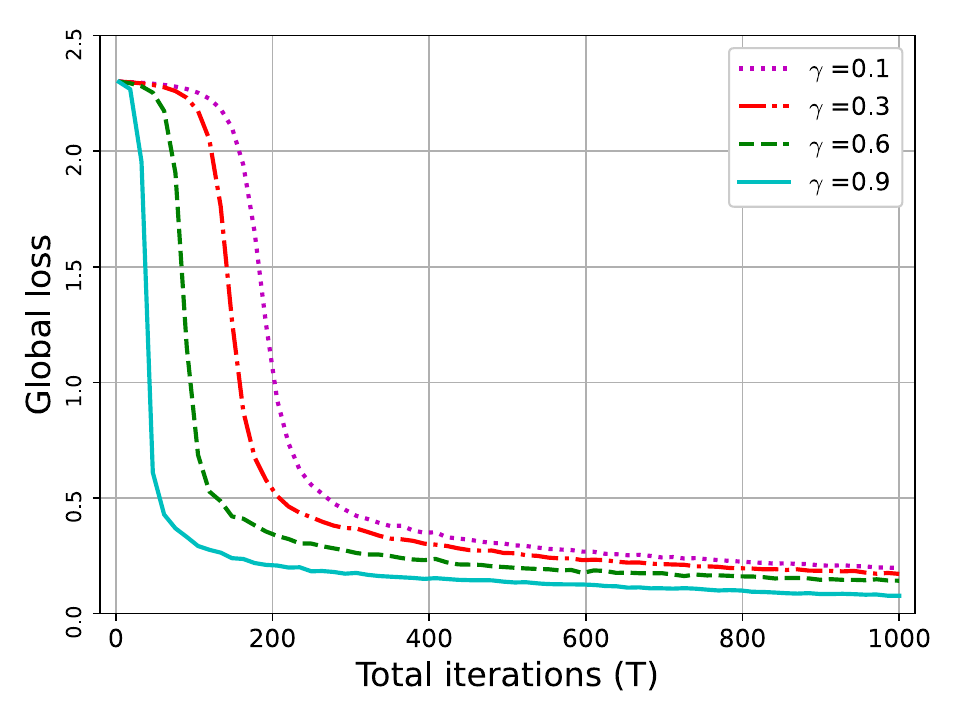}
    \end{minipage}
    \begin{minipage}[t]{0.24\textwidth}
        \centering
        \includegraphics[width=\textwidth]{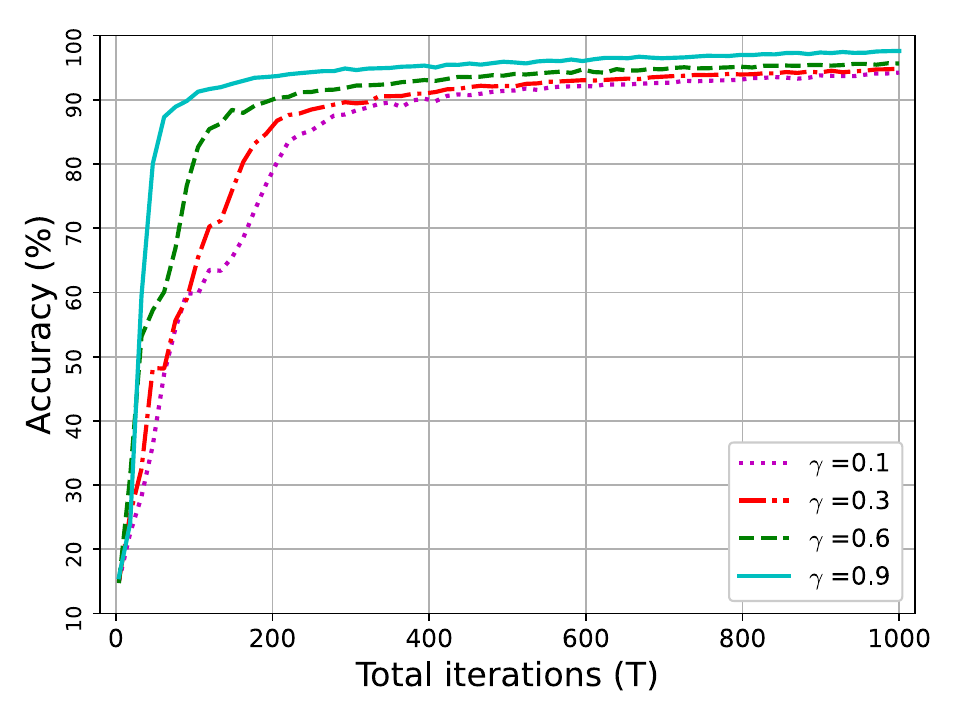}
    \end{minipage}\label{fig:gamma}}
    \subfigure[Global loss when $T$ reaches 500 and 1000 when $0<\gamma<1$]{
    \begin{minipage}[t]{0.24\textwidth}
        \centering
        \includegraphics[width=\textwidth]{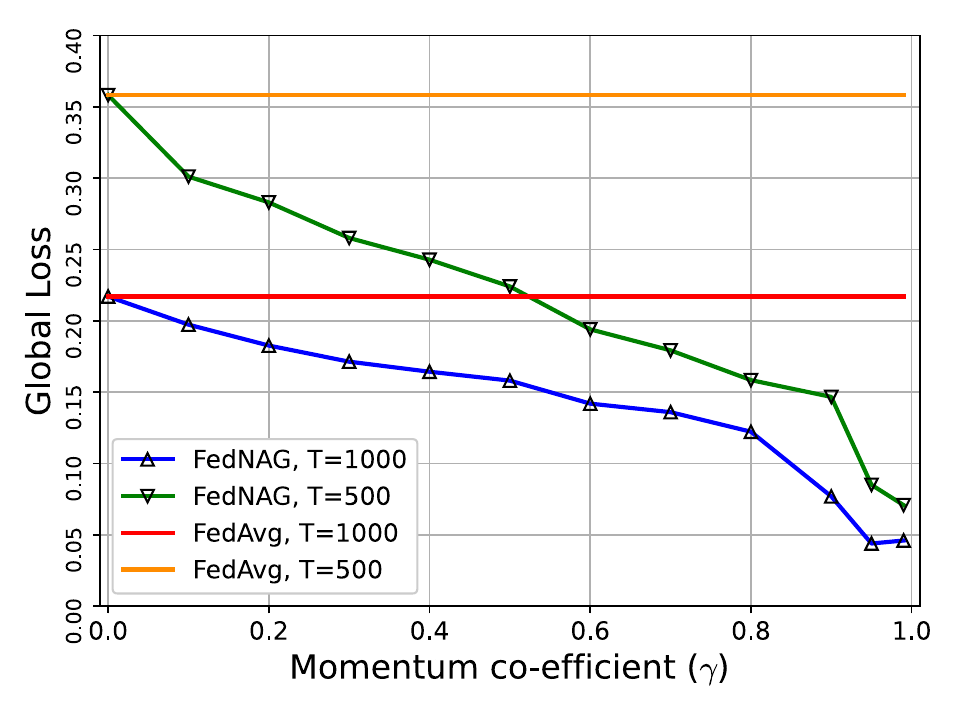}
    \end{minipage}
    \begin{minipage}[t]{0.24\textwidth}
        \centering
        \includegraphics[width=\textwidth]{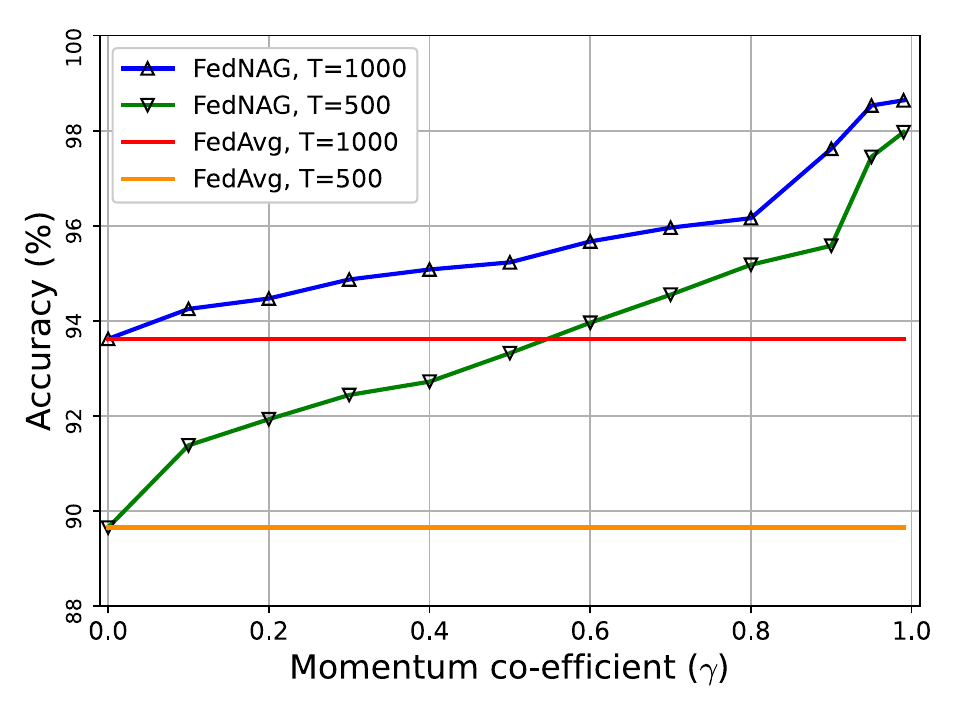}
    \end{minipage}\label{fig:gamma_T}}
    \subfigure[Effect of momentum coefficient $\gamma$ when $\gamma=1$]{
    \begin{minipage}[t]{0.24\textwidth}
        \centering
        \includegraphics[width=\textwidth]{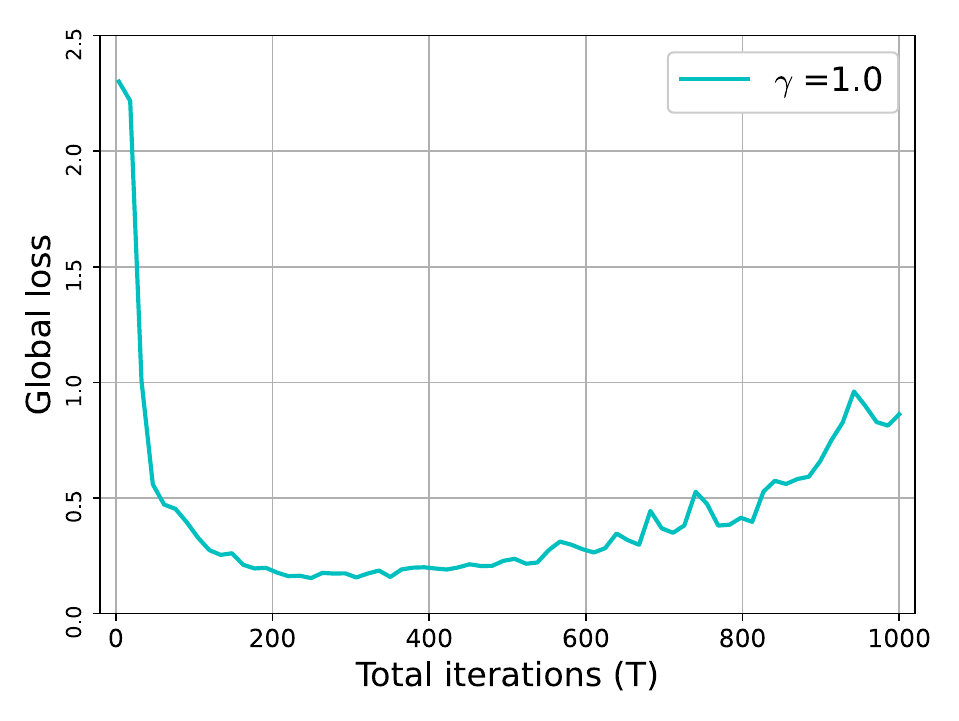}
    \end{minipage}
    \begin{minipage}[t]{0.24\textwidth}
        \centering
        \includegraphics[width=\textwidth]{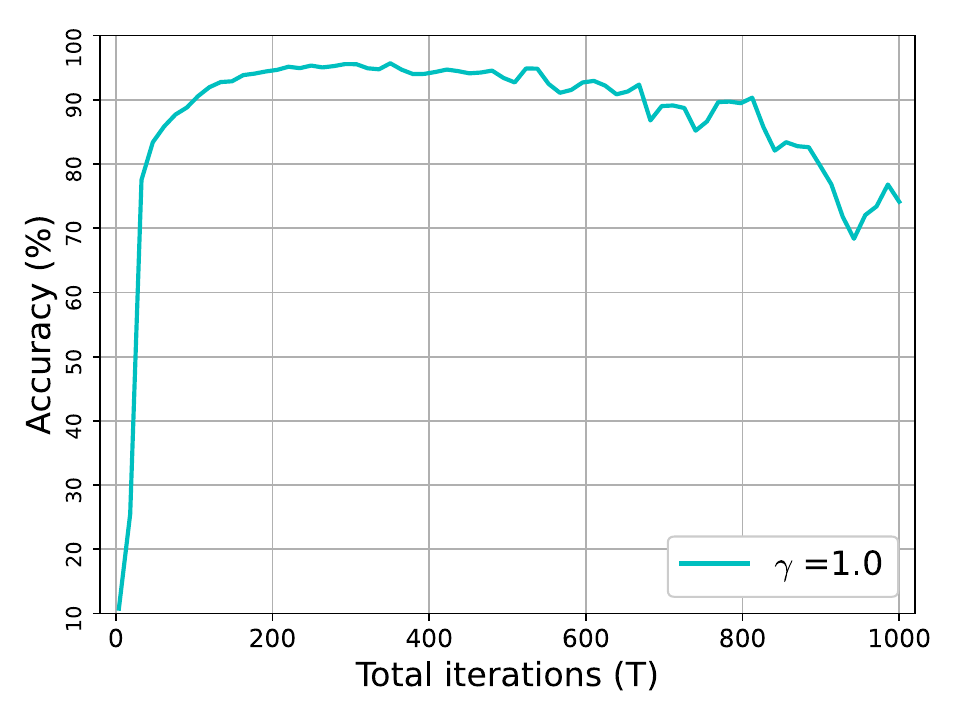}
    \end{minipage}\label{fig:gamma=1}}
    \subfigure[Effect of number of workers $N$]{
    \begin{minipage}[t]{0.24\textwidth}
        \centering
        \includegraphics[width=\textwidth]{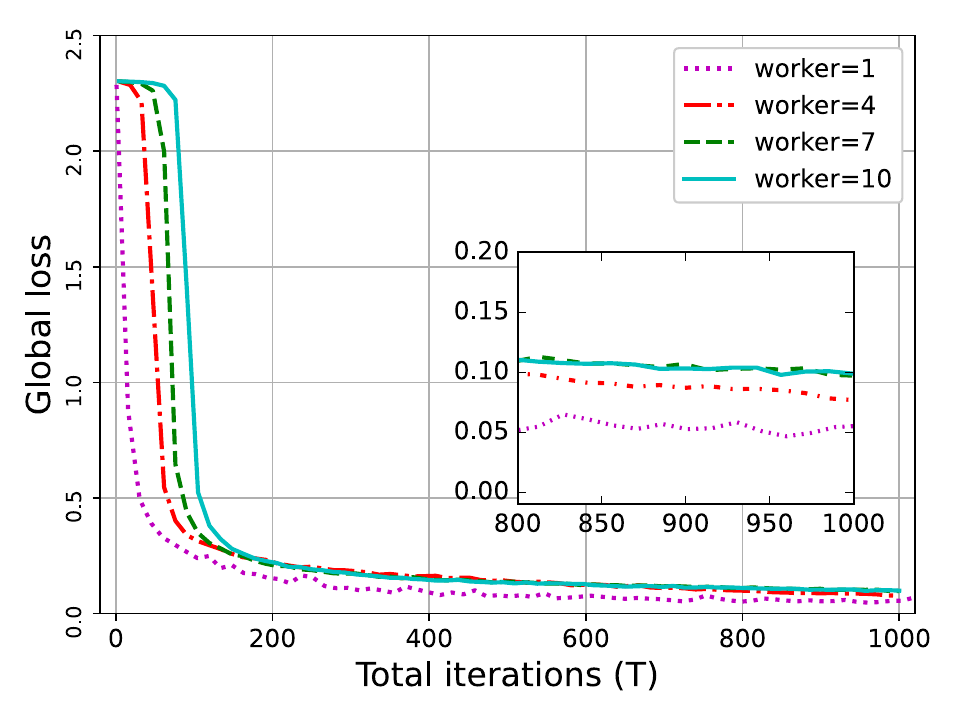}
    \end{minipage}
    \begin{minipage}[t]{0.24\textwidth}
        \centering
        \includegraphics[width=\textwidth]{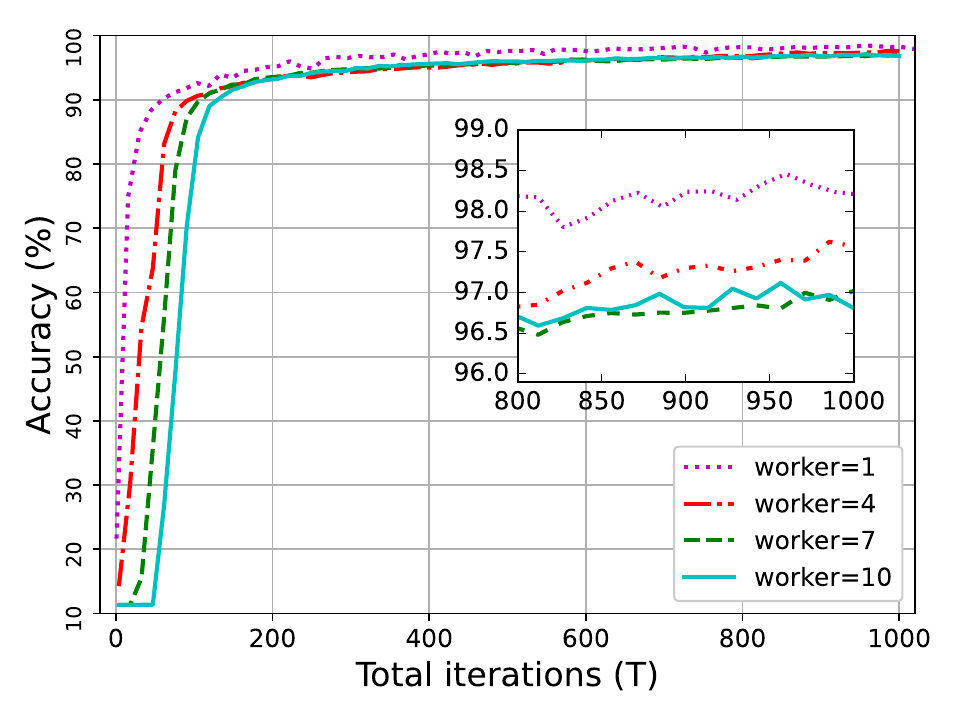}
    \end{minipage}\label{fig:worker}}
    \caption{Effect of $\tau, \gamma$, and $N$ when CNN trained on MNIST}
    \label{fig:compare}
\end{figure*}

\subsection{Experimental Setup}
\label{appendix:exp}
In order to evaluate the convergence performance of FedNAG, we employ three real-world datasets including MNIST, CIFAR-10 and CIFAR-100 for image classification. MNIST \cite{MNIST} contains gray-scale handwritten digits of images with 28 $\times$ 28 pixel. There are 60,000 images for training and 10,000 images for testing. CIFAR-10 \cite{CIFAR} contains 60,000 32x32 colour images in 10 classes, with 6000 images per class while CIFAR-100 \cite{CIFAR} contains 60,000 color images in 100 classes with 600 images per class. Both CIFAR-10 and CIFAR-100 contain 50,000 images for training and 10,000 images for testing. In our experiment, all samples are randomly distributed in each worker. It means that the data is i.i.d. distributed among workers. We will further consider the effects of non-i.i.d. data in Section \ref{sec:exp_non_IID}. We implement FedNAG and other benchmarks using PySyft library \cite{ryffel2018generic} based on the PyTorch framework. PySyft can emulate various virtual workers to process federated learning jobs. The training process is run on a GPU tower server with 4 NVIDIA GeForce RTX 2080Ti GPUs.

We use four models including linear regression model, logistic regression model, CNN model and DNN model. Linear regression  uses  mean  squared  error  loss,  and  logistic regression uses cross-entropy loss. The CNN model's structure is similar to the classic one in \cite{cnnModel}, which has two $5\times5$ convolutional layers with 32 and 64 channels respectively. In each convolutional layer, $2\times2$ max pooling is used. The last two following layers are ReLu activation and softmax. DNN model is VGG16. We use mini-batch in all experiments, and the batch size is 64. 
We set the default learning step size $\eta=0.01$. Other experimental settings are listed in Table ~\ref{tab:experiment_setting}.

\subsection{Performance Evaluation}
\subsubsection{Convergence Performance}
\label{sec:convergence}

In Fig.~\ref{fig:alg}, we compare the convergence performance of FedNAG with other four benchmark algorithms. The experiment is performed on three datasets. MNIST is trained by linear regression, logistic regression and CNN; CIFAR-10 is trained by CNN and DNN (VGG16), and CIFAR-100 is trained by DNN (VGG16). For convex model, we set $\tau=20, \gamma=0.9, N=4$. For non-convex model, we set $\tau=40, \gamma=0.9, N=4$. For MNIST, the total number of iterations $T$ is 1000. For CIFAR-10 and CIFAR-100, $T$ is set to 10000.

Figs.~\ref{fig:linear}, \ref{fig:logistic}, \ref{fig:cnn_mnist}, \ref{fig:cifar10}, \ref{fig:vgg16_cifar10}, and \ref{fig:vgg16_cifar100} show the values of the global loss function and accuracy trained under different models and datasets respectively. As a result, for convex models, we have cNAG $>$ FedNAG $>$ cSGD $>$ FedMom $>$ FedAvg. For non-convex models, we have cNAG $>$ FedNAG $>$ FedMom $>$ cSGD $>$ FedAvg. (We use ``$>$'' to indicate ``is better than'' for convenient presentation.) For centralized approaches, we can see cNAG performs better than cSGD in all cases. For distributed approaches, FedNAG also performs better than FedAvg and FedMom. It confirms that NAG is more advantageous compared with gradient decent for both centralized and FL environment. For cNAG and FedNAG, we can find FedNAG performs worse. This follows our expectation shown in Theorem \ref{theorem:Fwf-Fw*}.  FedNAG performs $\tau$ local updates before a global aggregation, causing less efficient updates and thus decreases the convergence performance. 
For FedNAG and FedMom, we can find FedNAG performs better than FedMom in all cases. For convex models, the gap between FedNAG and FedMom is significant. For non-convex models, FedNAG still performs better than FedMom. It confirms that FedNAG can accelerate the convergence performance for both convex and non-convex tasks, while FedMom only works well for non-convex tasks. 



Another interesting observation is that FedNAG can perform better than cSGD in all cases: The benefits of the momentum method can outweigh the performance loss by federated learning.

\subsubsection{Effects of Global Aggregation Frequency \texorpdfstring{$\tau$}{tau}}
In Fig.~\ref{fig:tau}, we evaluate the impact of $\tau$ based on global loss and accuracy using the same CNN model and MNIST dataset.  The setting for this experiment is $\gamma=0.5, T=1000, N=4$.

From Fig.~\ref{fig:tau}, we can observe when $\tau$ is increased, the convergence performance is reduced. When it converges, loss is larger and accuracy is lower. This matches \textcircled{1} of Theorem \ref{theorem:Fwf-Fw*}. The convergence upper bound increases with $\tau$. In Fig.~\ref{fig:tau_reach}, we observe the impact of $\tau$ in a wider range  $[5,640]$. In Fig.~\ref{fig:tau_reach} (left), we plot the number of iterations when the global loss reaches the target value 0.5. In Fig.~\ref{fig:tau_reach} (right), we plot the number of iterations when the accuracy reaches the target value 85\%. Since the global loss and accuracy may oscillate during the training process, the target global loss and accuracy may be reached several times. The red horizontal lines indicate the first and last iterations when the target values are reached, and the bar indicates the mean of the iterations  when the targets are reached.

The outcome also shows that given a targeted loss or accuracy, the number of iterations  does not monotonically increase or decrease with $\tau$. There is an optimal value $\tau$. This is because smaller $\tau$ leads to slower descent at the beginning (Fig.~\ref{fig:tau} in the main paper), but it converges closer to the optimal value in the end. The two effects cancel with each other and $\tau=40$ performs the best in our setting in Fig.~\ref{fig:tau_reach}. 

The similar phenomenon also appears in other FL algorithms~\cite{wang2019slowmo, wang2019adaptive}. Moreover, If we double $\tau$ when $\tau$ is small, the number of iterations to reach the targets does not change much. However, if we double $\tau$ when $\tau$ is too large (e.g., $\tau\geq 80$), then the number of iterations to reach the targets substantially increases. This matches \textcircled{5} of Theorem \ref{theorem:wt-wkt}, which concludes that larger $\tau$ leads to exponential increase of $h(\cdot)$. Therefore, increasing $\tau$ will  significantly delay the training process when $\tau$ is too large.

\begin{figure*}
    \centering
    \begin{minipage}[t]{0.3\textwidth}
        \centering
        \includegraphics[width=\textwidth]{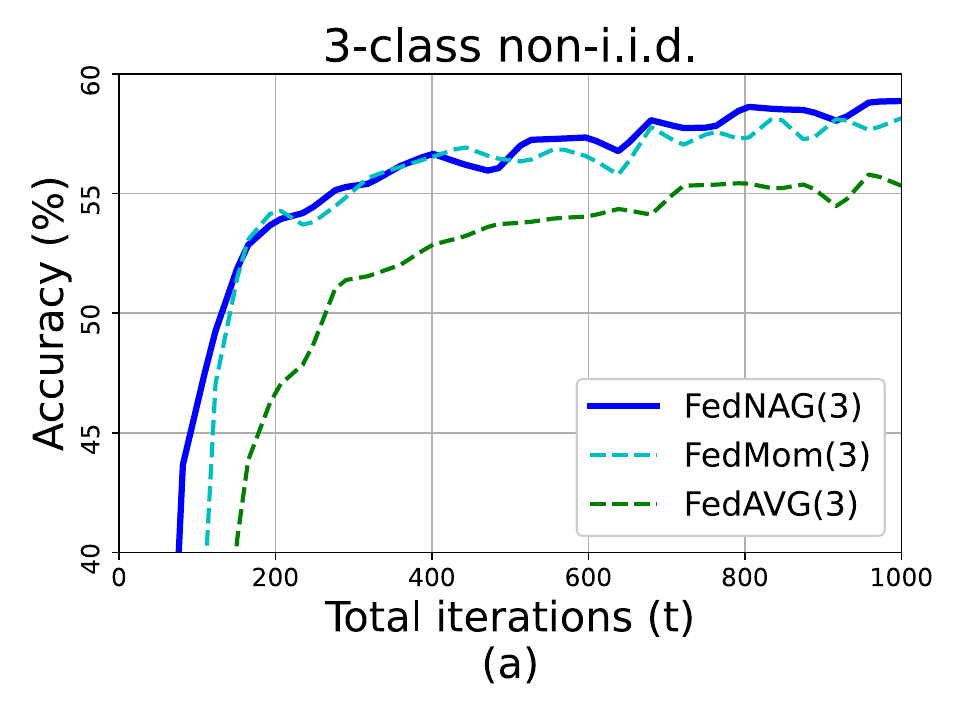}
    \end{minipage}
    \begin{minipage}[t]{0.3\textwidth}
        \centering
        \includegraphics[width=\textwidth]{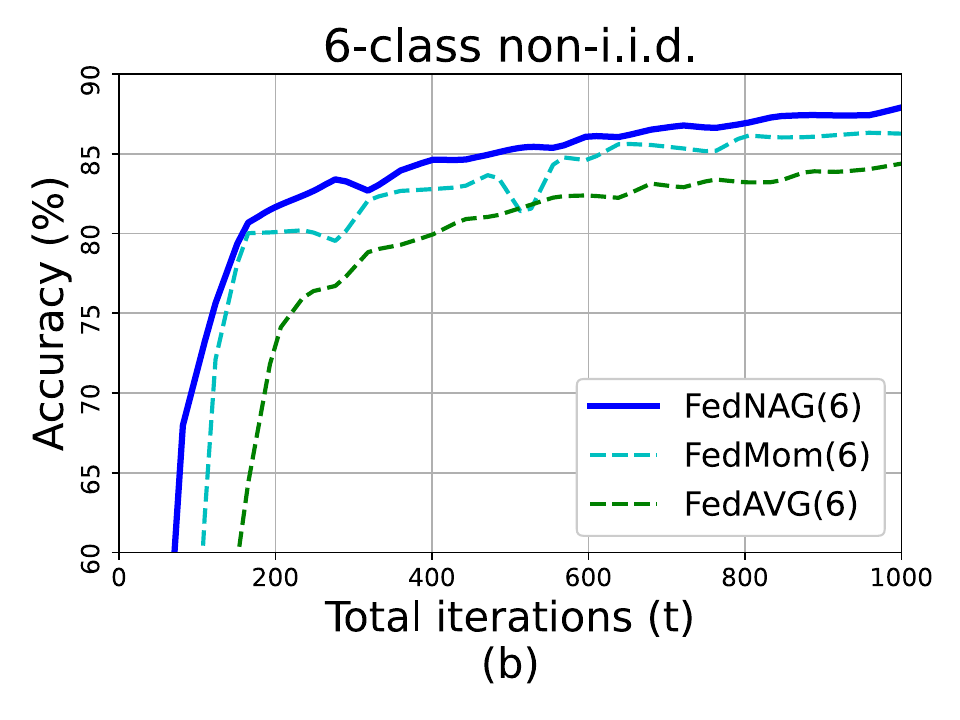}
    \end{minipage}
    \begin{minipage}[t]{0.3\textwidth}
        \centering
        \includegraphics[width=\textwidth]{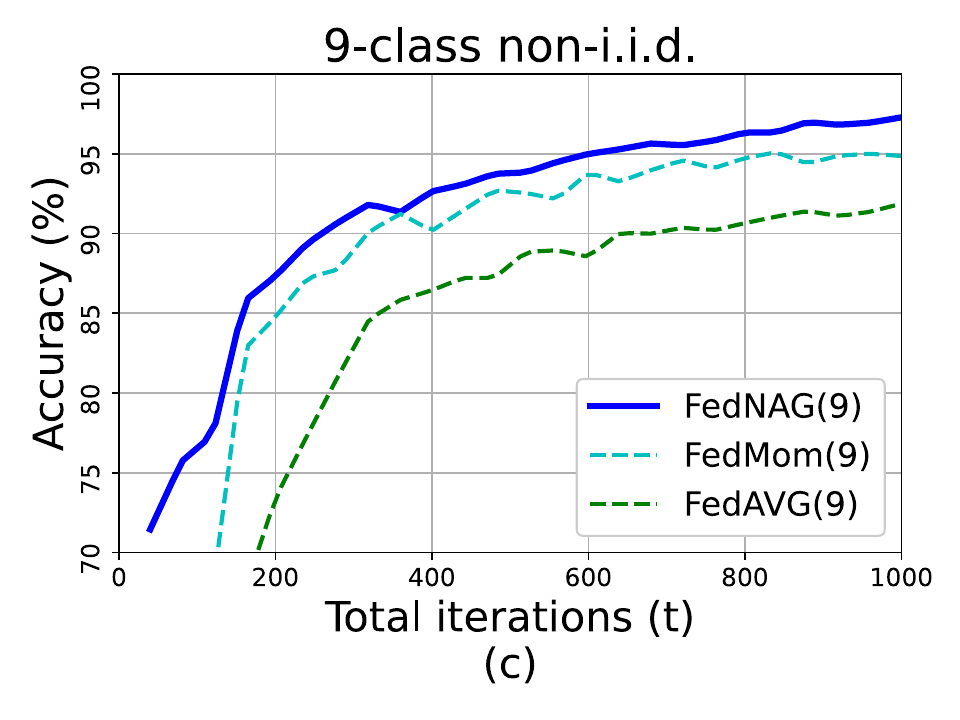}
    \end{minipage}
    \caption{Performance comparison under 3-class (a), 6-class (b), and 9 class (c) non-i.i.d. data distribution.}
    \label{fig:noniid}
\end{figure*}

\begin{figure}
    \centering
    \begin{minipage}[t]{0.24\textwidth}
        \centering
        \includegraphics[width=\textwidth]{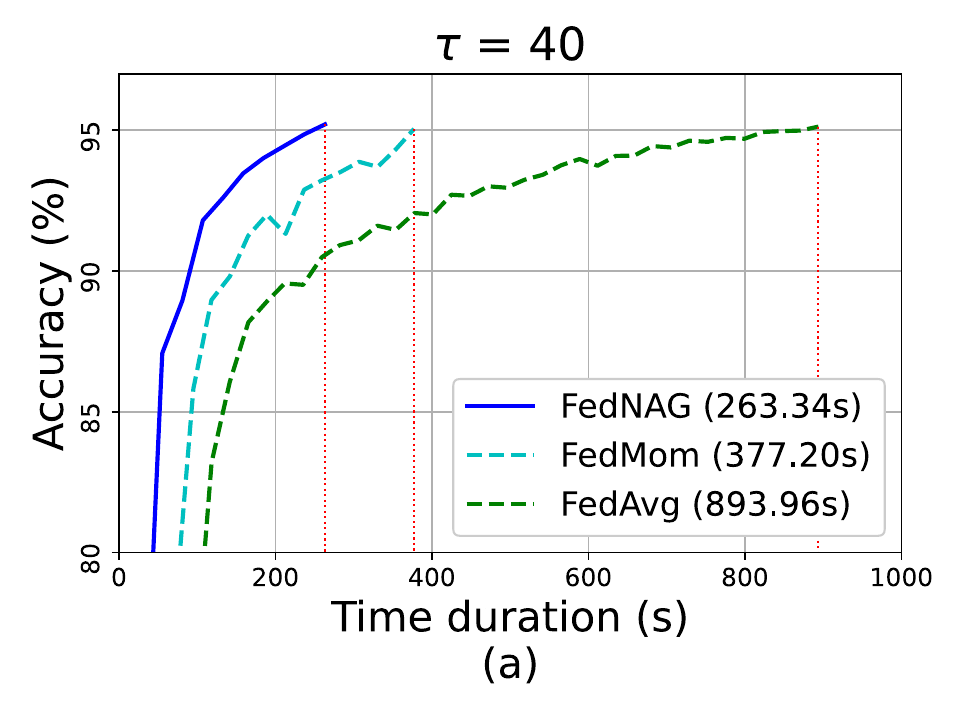}
    \end{minipage}
    \begin{minipage}[t]{0.24\textwidth}
        \centering
        \includegraphics[width=\textwidth]{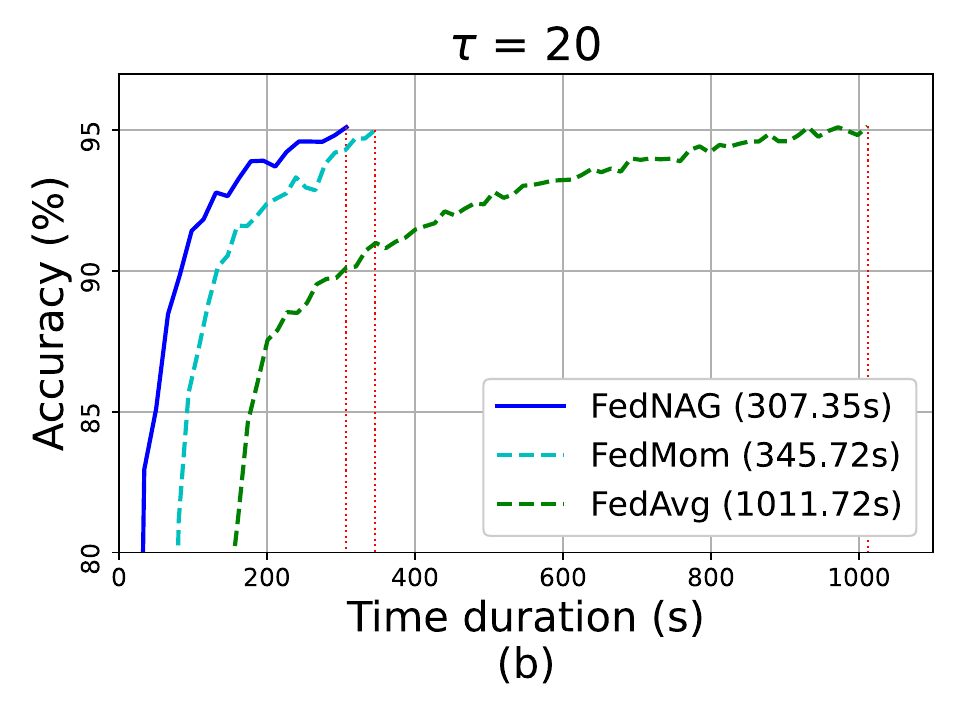}
    \end{minipage}
    \caption{Comparison of total training time to reach  0.95 accuracy for $\tau=40$ (a) or $\tau=20$ (b) when CNN is trained on MNIST. The time to reach 0.95 accuracy is labeled in the legends.}
    \label{fig:trace_driven}
\end{figure}

\subsubsection{Effects of Momentum Coefficient \texorpdfstring{$\gamma$}{gamma}}
In Fig.~\ref{fig:gamma}, we evaluate the effects of $\gamma$. The setting for this experiment is $\tau=4, T=1000, N=4$. We also use the same CNN model trained on the same MNIST dataset. 

Fig.~\ref{fig:gamma} shows the global loss and accuracy under $\gamma=0.1, 0.3, 0.6, 0.9$ respectively. It shows that $\gamma$ can increase the convergence performance (smaller global loss value and higher accuracy). For Fig.~\ref{fig:gamma_T}, we evaluate the global loss at $T=500$ and $T=1000$ respectively, when  $\gamma$ ranges from $[0,0.99]$. Two horizontal lines are the benchmarks where only FedAvg is used. For both $T=500$ and $T=1000$, we can see the global loss decreases when $\gamma$ is getting large. Accuracy is also increased at the same time. However, from Fig.~\ref{fig:gamma=1}, when $\gamma=1$, the global loss cannot converge due to the prerequisite where $0<\gamma<1$ \cite{ruder2016overview}.

\subsubsection{Effects of Number of Workers \texorpdfstring{$N$}{N}}

In Fig.~\ref{fig:worker}, we evaluate the global loss and accuracy based on different number of workers $N$ using the same CNN model and MNIST dataset. The experiment setting is $\tau =4, \gamma=0.9, T=1000$.  
From Fig.~\ref{fig:worker}, we can see that increasing $N$ will cause a decline of convergence performance. This follows our expectation because more workers cause more divergence among the workers and thus decrease convergence performance. However, after a sufficient number of iterations, the global loss and accuracy with more workers will be closer to those with fewer workers. It shows that FedNAG is applicable when there are more workers in the system.

\subsection{Effects of non-i.i.d. data distribution}\label{sec:exp_non_IID}

In Fig.~\ref{fig:noniid}, we evaluate the effects of different levels of non-i.i.d. data distribution. 
We train CNN on MNIST with the setting $\tau=40, \gamma =0.9, N=4, T=1000$. The curves show the training accuracy. To quantify the level of non-i.i.d. data distribution, we explicitly assign only $x<10$ out of 10 classes of data for each worker. (Each worker has data samples from a subset of classes.) Smaller \textit{x} represents higher level of non-i.i.d. setting. We use \textit{3-class non-i.i.d.}, \textit{6-class non-i.i.d.}, and \textit{9-class non-i.i.d.} to represent high, middle and low level of non-i.i.d. data respectively.

In all cases, We can observe FedNAG $>$ FedMom $>$ FedAvg. This shows that FedNAG outperforms benchmarks under any levels of non-i.i.d. data distribution.
We also observe higher level of non-i.i.d. setting decreases convergence performance for all algorithms. Specifically, FedNAG achieves 58.87\% accuracy for high level non-i.i.d. data, while achieving 87.90\% accuracy and 97.28\% accuracy for middle and low level non-i.i.d. data respectively. This matches our expectations in Observation \textcircled{6} in Theorem \ref{theorem:wt-wkt} where higher level of non-i.i.d. setting causes more data divergence that is denoted by larger $\delta$, and therefore lowers the accuracy.

\subsection{Trace-Driven Simulation}

We emulate the real-world edge-computing environment and compare the total training time with an expected learning accuracy (0.95) on FedNAG and other benchmarks when CNN in trained on MNIST.  We train the model in the GPU tower server and keep the trace of the sequence of iterations. We then use real-world devices as workers (one laptop with Intel Core i3 M380 CPU, three Android phones: Nubia z17s with Qualcomm Snapdragon 835 CPU, Realme GT Neo with MTK Dimensity 1200 CPU, Redmi K30 Ultra with MTK Dimensity 1000+ CPU) to sample worker computing delays. The aggregator computing delay is sampled on GPU tower server. All devices are connected to HUAWEI honor router X2+ with 5GHz WIFI to sample communication delays which depend on the communication workload of transmitting the models/momenta. We use the trace of the sequence of iterations and the sampled delays to figure out the overall training time as if the training process is conducted by the GPU server (aggregator) and the four devices. 

Please note such approach to use a digital representation of physical objects to conduct experiment is widely used in distributed systems, IoT, and machine learning applications \cite{teng2021recent,kirchhof2021understanding}. It can generate the convincing system performance evaluation without deploying physical devices.

In Fig.~\ref{fig:trace_driven}, we observe that under two different settings ($\tau=40$ or $\tau=20$, $N=4$), to reach the accuracy 0.95, FedNAG spends 263.34s ($\tau=40$) and 307.35s ($\tau=20$) while FedMom spends 377.20s ($\tau=40$) and 345.72s ($\tau=20$), and FedAvg spends 893.96s ($\tau=40$) and 1011.72s ($\tau=20$). This demonstrates that FedNAG is efficient and decreases the total training time by 11--70\% compared with FedMom and FedAvg.

%% file: 6conclusion.tex
In this paper, we focus on  FedNAG, a NAG style momentum-based FL algorithm. FedNAG allows each worker to update its weights and momenta by its local dataset for a number of local iterations between two global aggregations. On the global aggregation step, the aggregator collects and averages the weights and momenta from all workers and distributes them to the workers. 
The convergence analysis shows the upper bound of the gap between the global loss function derived by FedNAG at iteration $T$ and the optimal solution. We compare FedNAG and FedAvg, and conclude that as long as the learning step size is sufficiently small, FedNAG outperforms FedAvg. 
Experiments based on real-world datasets and trace-driven simulation are conducted, demonstrating that FedNAG increases the learning accuracy by 3--24\% and decreases the total training time by 11--70\% compared with the benchmarks.

%% file: appendix.tex
\input{appendixA}
\input{appendixB}
\input{appendixC}
\input{appendixD}
\input{appendixE}
\input{appendixF}

%% file: appendixA.tex
\section{FedNAG vs. Centralized NAG (Observation \texorpdfstring{\textcircled{4}}{4} in Theorem \ref{theorem:wt-wkt})} \label{appendixA}
\begin{prop} \label{prop1}
When $\tau = 1$, FedNAG is equivalent to centralized NAG. The update rules of FedNAG yield as follows:
\begin{align*}
    \mathbf{v}(t) &= \gamma\mathbf{v}(t-1) - \eta\nabla F(\mathbf{w}(t-1)),\\
        \mathbf{w}(t) 
        & = \mathbf{w}(t-1) - \gamma\mathbf{v}(t-1) + (1+\gamma)\mathbf{v}(t) \\
        & = \mathbf{w}(t-1) + \gamma\mathbf{v}(t) - \eta\nabla F(\mathbf{w}(t-1)).
\end{align*}
\end{prop}
\begin{proof}
When $\tau = 1$, we have $\mathbf{v}_i(t) = \mathbf{v}(t)$ and $\mathbf{w}_i(t) = \mathbf{w}(t)$ for all $t$. Thus,
\begin{align*}
    \mathbf{v}(t)
    & = \frac{\sum_{i=1}^N D_i \mathbf{v}_i(t)}{D}& \\
    & = \frac{\sum_{i=1}^N D_i(\gamma\mathbf{v}_i(t-1) - \eta\nabla F_i(\mathbf{w}_i(t-1)))}{D}& \\
    & = \gamma\mathbf{v}(t-1) - \eta\frac{\sum_{i=1}^N D_i \nabla F_i(\mathbf{w}(t-1))}{D}& \\
    & = \gamma\mathbf{v}(t-1) - \eta\nabla F(\mathbf{w}(t-1)),
\end{align*}
where the last term in the last equality is because
\begin{displaymath}
    \frac{\sum_{i=1}^N D_i \nabla F_i(\mathbf{w})}{D} = \nabla \left(\frac{\sum_{i=1}^N D_i F_i(\mathbf{w})}{D}\right) = \nabla F(\mathbf{w})
\end{displaymath}
based on the linearity of the gradient operator. Then,
\begin{align*}
    \mathbf{w}(t) 
    & = \frac{\sum_{i=1}^N D_i\mathbf{w}_i(t)}{D}& \\
    & = \frac{\sum_{i=1}^N D_i(\mathbf{w}_i(t-1)-\gamma\mathbf{v}_i(t-1)+(1+\gamma)\mathbf{v}_i(t))}{D}& \\
    & = \mathbf{w}(t-1)-\gamma\mathbf{v}(t-1)+(1+\gamma)\mathbf{v}(t).
\end{align*}
Therefore, Proposition \ref{prop1} has been proven.
\end{proof}

%% file: appendixB.tex
\section{Proof of Theorem \ref{theorem:wt-wkt}} \label{appendixB}
To prove Theorem \ref{theorem:wt-wkt}, the progress mainly includes four steps. (1) We first introduce an important equality in Lemma \ref{lemma:sequence}, which will be used later. (2) We bound $\Vert \mathbf{w}_i(t) - \mathbf{w}_{[k]}(t) \Vert$ in Lemma \ref{lemma:w_it-w_kt} based on Lemma \ref{lemma:sequence}. (3) Based on the result of Lemma \ref{lemma:w_it-w_kt}, we then bound $\Vert \mathbf{v}(t) - \mathbf{v}_{[k]}(t) \Vert$ in Lemma \ref{lemma:vt-vkt}. (4) Finally, based on the result of Lemma \ref{lemma:vt-vkt}, we bound $\Vert \mathbf{w}(t) - \mathbf{w}_{[k]}(t) \Vert$, which concludes Theorem \ref{theorem:wt-wkt}.

\begin{lemma} \label{lemma:sequence}
Given
\begin{align}
&a_{t}=\frac{\delta_{i}}{\beta}\left(\frac{\frac{1+\eta\beta+\eta\beta\gamma}{\gamma}-B}{A-B} A^{t}-\frac{\frac{1+\eta\beta+\eta\beta\gamma}{\gamma}-A}{A-B} B^{t}\right),\\
&A+B=\frac{1+\eta\beta+\eta\beta\gamma+\gamma}{\gamma}=\frac{(1+\eta\beta)(1+\gamma)}{\gamma},\\
&AB=\frac{1+\eta\beta}{\gamma},
\end{align}
where $t=0,1,2,...,0<\gamma<1,\eta\beta>0$, we have
\begin{align}
    (1+\eta\beta)a_{t-1}+\eta\beta\gamma\sum_{i=0}^{t-1} a_i = \gamma a_t.
\end{align}
\begin{proof} [Proof of Lemma \ref{lemma:sequence}]
For convenience, we define
\begin{align*}
     C& \triangleq \frac{\frac{1+\eta\beta+\eta\beta\gamma}{\gamma}-B}{A-B}=\frac{A-1}{A-B},\\
     D& \triangleq \frac{A-\frac{1+\eta\beta+\eta\beta\gamma}{\gamma}}{A-B}=\frac{1-B}{A-B}.
\end{align*}
Therefore,
\begin{align*}
    a_t=\frac{\delta_i}{\beta}(CA^t+DB^t).
\end{align*}
According to the inverse theorem of Vieta's formulas, we have
\begin{align} \label{eq:quadratic}
    \gamma x^2-(1+\eta\beta+\eta\beta\gamma+\gamma)x+\eta\beta+1=0,
\end{align}
where $x$ values are the roots of quadratic equation.
Here, the discriminant of the quadratic equation is positive.
\begin{align*} 
\Delta&=(1+\eta\beta+\eta\beta\gamma+\gamma)^2-4(1+\eta\beta)\gamma\\
&>(1+\eta\beta+\gamma)^2-4(1+\eta\beta)\gamma\\
&=((1+\eta\beta)-\gamma)^2 >0.
\end{align*}
Thus, $A$ and $B$ (roots) can be expressed as follows:
\begin{align}
&A=\frac{(1+\eta\beta)(1+\gamma)+\sqrt{(1+\eta\beta)^2(1+\gamma)^{2}-4\gamma(1+\eta\beta)}}{2\gamma},\\
&B=\frac{(1+\eta\beta)(1+\gamma)-\sqrt{(1+\eta\beta)^2(1+\gamma)^{2}-4\gamma(1+\eta\beta)}}{2\gamma}.
\end{align}
Then we have
\begin{flalign*}
&(1+\eta\beta) a_{t-1}+\eta \beta \gamma \sum_{i=0}^{t-1} a_{i} - \gamma a_{t}&\\
=&(1+\eta\beta)\frac{\delta_{i}}{\beta}\left(C A^{t-1}+D B^{t-1}\right)+\eta \beta \gamma  \frac{\delta_{i}}{\beta} C \frac{A^{t}-1}{A-1}&\\
&+\eta \beta \gamma \frac{\delta_{i}}{\beta} D \frac{B^{t}-1}{B-1}-\gamma \frac{\delta_{i}}{\beta} C A^{t}-\gamma \frac{\delta_{i}}{\beta} D B^{t}&\\
=&\frac{\delta_{i}}{\beta}\left[\frac{A^{t-1} C}{1-A}\left(\gamma A^{2}-(1+\eta \beta+ \eta \beta \gamma+ \gamma) A+1+\eta \beta \right)\right.&\\
&\left.+\frac{B^{t-1} D}{1-B}\left(\gamma B^{2}-(1+\eta \beta+ \eta \beta \gamma+ \gamma) B+1+\eta \beta \right)\right]&\\
&-\frac{\delta_{i}}{\beta}\eta \beta\gamma\left(\frac{C}{A-1}+\frac{D}{B-1}\right)&\\
=&0-\eta \delta_{i}\gamma\left(\frac{C}{A-1}+\frac{D}{B-1}\right)&\\
=&0.\\
&\rightline{\text{(because $A$, $B$ satisfy \eqref{eq:quadratic})}}.
\end{flalign*}
Thus, Lemma \ref{lemma:sequence} has been proven.
\end{proof}
\end{lemma}
\subsection{Bounding \texorpdfstring{$\Vert\mathbf{w}_i(t)-\mathbf{w}_{[k]}(t) \Vert$}{||wi(t)-w[k](t)||}}
To prove Lemma \ref{lemma:w_it-w_kt}, the progress mainly includes two steps. (1) We first bound the gap of  $\Vert\mathbf{v}_i(t)-\mathbf{v}_{[k]}(t) \Vert$. (2) Then we bound the gap of $\Vert\mathbf{w}_i(t)-\mathbf{w}_{[k]}(t) \Vert$, which concludes Lemma \ref{lemma:w_it-w_kt}.
\begin{lemma} \label{lemma:w_it-w_kt}
For any interval $[k]$, $\forall t \in [(k-1)\tau,k\tau]$, we have
\begin{align}
    \Vert\mathbf{w}_i(t)-\mathbf{w}_{[k]}(t) \Vert \leq f_i(t-(k-1)\tau),
\end{align}
where we define the function $f_i(x)$ as
\begin{align}
    f_i(x) \triangleq \frac{\delta_i}{\beta}(\gamma^x(CA^x+DB^x)-1).
\end{align}
\end{lemma}
\begin{proof} [Proof of Lemma \ref{lemma:w_it-w_kt}]
When $t=(k-1)\tau$, we know $\mathbf{w}_i(t)=\mathbf{w}(t)=\mathbf{w}_{[k]}(t)$ by the definition of $\mathbf{w}_{[k]}(t)$ and aggregation rules. Hence, we have $\Vert\mathbf{w}_i(t)-\mathbf{w}_{[k]}(t) \Vert =0$. Meanwhile, when $t=(k-1)\tau$, $x=0$ and $f_i(0)=0$. Thus, Lemma \ref{lemma:w_it-w_kt} holds.

When $t \in ((k-1)\tau, k\tau]$, we bound the momentum gap
\begin{flalign} \label{eq:v_i-v_k}
    & \Vert \mathbf{v}_i(t) - \mathbf{v}_{[k]}(t) \Vert\nonumber \\
     =& \Vert \gamma\mathbf{v}_i(t-1) - \eta \nabla F_i(\mathbf{w}_i(t-1))\nonumber\\ 
    &-(\gamma\mathbf{v}_{[k]}(t-1) - \eta \nabla F(\mathbf{w}_{[k]}(t-1))) \Vert\nonumber \\
     =& \Vert \gamma(\mathbf{v}_i(t-1) - \mathbf{v}_{[k]}(t-1)) - \eta [ \nabla F_i(\mathbf{w}_i(t-1))- \nonumber\\
    & \nabla F_i(\mathbf{w}_{[k]}(t-1)) + \nabla F_i(\mathbf{w}_{[k]}(t-1)) - \nabla   F(\mathbf{w}_{[k]}(t-1)) ] \Vert \nonumber\\
    &\rightline{\text{(adding a zero term)}}\nonumber \\
    \leq& \gamma\Vert \mathbf{v}_i(t-1) - \mathbf{v}_{[k]}(t-1) \Vert \nonumber\\ 
    &+ \eta \Vert \nabla F_i(\mathbf{w}_i(t-1))- \nabla F_i(\mathbf{w}_{[k]}(t-1)) \Vert \nonumber\\
    &+ \eta \Vert \nabla F_i(\mathbf{w}_{[k]}(t-1))-\nabla F(\mathbf{w}_{[k]}(t-1)) \Vert \nonumber\\
    &\rightline{\text{(from triangle inequality)}} \nonumber\\
    \leq& \gamma\Vert \mathbf{v}_i(t-1) - \mathbf{v}_{[k]}(t-1) \Vert \nonumber \\ 
    &+\eta\beta \Vert \mathbf{w}_i(t-1) - \mathbf{w}_{[k]}(t-1) \Vert + \eta\delta_i. \\
    &\rightline{\text{(from $\beta$-smoothness and \eqref{definition:GD})}}\nonumber
\end{flalign}
We use $\gamma^0, \gamma^1,\dots, \gamma^{t-(k-1)\tau-1}$ as multipliers to multiply \eqref{eq:v_i-v_k} when $t, t-1,\dots, (k-1)\tau+1$, respectively.
\begin{flalign*}
    & \Vert \mathbf{v}_i(t) - \mathbf{v}_{[k]}(t) \Vert 
      \leq \gamma\Vert \mathbf{v}_i(t-1) - \mathbf{v}_{[k]}(t-1) \Vert &\\ 
    & \quad +\eta\beta \Vert \mathbf{w}_i(t-1) - \mathbf{w}_{[k]}(t-1) \Vert + \eta\delta_i,& \\
    & \gamma\Vert \mathbf{v}_i(t-1) - \mathbf{v}_{[k]}(t-1) \Vert 
      \leq \gamma(\gamma\Vert \mathbf{v}_i(t-2) - \mathbf{v}_{[k]}(t-2) \Vert&  \\ 
    & \quad +\eta\beta \Vert \mathbf{w}_i(t-2) - \mathbf{w}_{[k]}(t-2) \Vert + \eta\delta_i),& \\
    &\dots& \\
    & \gamma^{t-(k-1)\tau-1}\Vert \mathbf{v}_i((k-1)\tau+1) - \mathbf{v}_{[k]}((k-1)\tau+1) \Vert& \\ 
    &\leq \gamma^{t-(k-1)\tau-1}(\gamma\Vert \mathbf{v}_i((k-1)\tau) - \mathbf{v}_{[k]}((k-1)\tau) \Vert&  \\ 
    &\quad+\eta\beta \Vert \mathbf{w}_i((k-1)\tau) - \mathbf{w}_{[k]}((k-1)\tau) \Vert + \eta\delta_i).&
\end{flalign*}
For convenience, we define $G_i(t) \triangleq \Vert \mathbf{w}_i(t)-\mathbf{w}_{[k]}(t)\Vert$. Summing up all of the above inequalities by integer $j\in [1, t-(k-1)\tau]$, we have
\begin{flalign*}
&\Vert\mathbf{v}_{i}(t)-\mathbf{v}_{[k]}(t) \Vert \\
\leq&\eta\beta\sum_{j=1}^{t-(k-1)\tau}\gamma^{j-1} G_i(t-j)+\eta\delta_{i}\sum_{j=1}^{t-(k-1)\tau}\gamma^{j-1}\\
&+\gamma^{t-(k-1) \tau}\Vert\mathbf{v}_{i}((k-1) \tau)-\mathbf{v}_{[k]}((k-1) \tau)\Vert.
\end{flalign*}
When $t=(k-1)\tau$, we know $\mathbf{v}_i(t)=\mathbf{v}(t)=\mathbf{v}_{[k]}(t)$ by the definition of $\mathbf{v}_{[k]}(t)$ and aggregation rules. Then we have $\Vert\mathbf{v}_{i}((k-1) \tau)-\mathbf{v}_{[k]}((k-1) \tau)\Vert =0$, so that the last term of above inequality is zero and
\begin{flalign}
\label{ieq:vit-vkt}
&\Vert\mathbf{v}_{i}(t)-\mathbf{v}_{[k]}(t) \Vert \nonumber\\
\leq& \eta\beta\sum_{j=1}^{t-(k-1)\tau}\gamma^{j-1} G_i(t-j)+\eta\delta_{i}\sum_{j=1}^{t-(k-1)\tau}\gamma^{j-1}.
\end{flalign}
Now, we can bound the gap between  $\mathbf{w}_i(t)$ and $\mathbf{w}_{[k]}(t)$. When $t \in ((k-1)\tau, k\tau]$, we have
\begin{flalign}
\label{eq:wit-wkt}
& \Vert \mathbf{w}_i(t) - \mathbf{w}_{[k]}(t) \Vert \nonumber\\
=& \Vert\mathbf{w}_i(t-1) + \gamma\mathbf{v}_i(t) - \eta \nabla F_i(\mathbf{w}_i(t-1)) \nonumber\\
&- (\mathbf{w}_{[k]}(t-1) + \gamma\mathbf{v}_{[k]}(t) - \eta \nabla F(\mathbf{w}_{[k]}(t-1)))\nonumber\Vert\\
&\rightline{\text{(from \eqref{eq:localWi} and \eqref{eq:wkt=})}}\nonumber\\
=& \Vert\mathbf{w}_i(t-1)- \mathbf{w}_{[k]}(t-1)+ \gamma(\mathbf{v}_i(t) -\mathbf{v}_{[k]}(t)) \nonumber\\
&- \eta [\nabla F_i(\mathbf{w}_i(t-1))-\nabla F_i(\mathbf{w}_{[k]}(t-1))\nonumber\\
&+\nabla F_i(\mathbf{w}_{[k]}(t-1))- \nabla F(\mathbf{w}_{[k]}(t-1))]\Vert\nonumber\\
&\rightline{\text{(adding a zero term)}}\nonumber\\
\leq& \Vert \mathbf{w}_i(t-1) - \mathbf{w}_{[k]}(t-1) \Vert+\gamma\Vert \mathbf{v}_i(t) - \mathbf{v}_{[k]}(t) \Vert \nonumber\\ 
&+\eta\beta \Vert \mathbf{w}_i(t-1) - \mathbf{w}_{[k]}(t-1) \Vert + \eta\delta_i \nonumber\\
&\rightline{\text{(from triangle inequality, $\beta$-smoothness and \eqref{definition:GD})}}\nonumber\\
=&(\eta\beta+1)\Vert \mathbf{w}_i(t-1) - \mathbf{w}_{[k]}(t-1) \Vert\nonumber\\
&+\gamma\Vert \mathbf{v}_i(t) - \mathbf{v}_{[k]}(t) \Vert+\eta\delta_i.
\end{flalign}
Substituting inequality \eqref{ieq:vit-vkt} into \eqref{eq:wit-wkt} and using $G_i(t)$ to denote $\Vert \mathbf{w}_i(t)-\mathbf{w}_{[k]}(t)\Vert$ for $t, t-1,\cdots, (k-1)\tau+1$, we have
\begin{flalign}\label{ieq:Gitleq}
G_i(t)\leq&(\eta\beta+1)G_i(t-1)+\eta\beta\gamma\sum_{j=1}^{t-(k-1)\tau}\gamma^{j-1} G_i(t-j)&\nonumber\\
&+\eta\delta_{i}\gamma\sum_{j=1}^{t-(k-1)\tau}\gamma^{j-1}+\eta\delta_i&\nonumber\\
=&(\eta\beta+1)G_i(t-1)+\eta\beta\gamma\sum_{j=1}^{t-(k-1)\tau}\gamma^{j-1} G_i(t-j)&\nonumber\\
&+\eta\delta_{i}\sum_{j=0}^{t-(k-1)\tau}\gamma^{j}.&
\end{flalign}

For convenience, we define $g_i(x)\triangleq\frac{\delta_i}{\beta}(CA^x+DB^x)$, where $A$ and $B$ are defined in Theorem~\ref{theorem:wt-wkt}; $C$ and $D$ are defined in Lemma~\ref{lemma:sequence}. We have
\begin{align}
\label{eq:fix}
    f_i(x)=\gamma^x  g_i(x)-\frac{\delta_i}{\beta}.
\end{align}

Next, we use induction to prove $G_i(t)\leq f_i(t-(k-1)\tau)$. For the induction, we assume that
\begin{align}\label{ieq:Gip}
    G_i(p)\leq f_i(p-(k-1)\tau)
\end{align}
holds for some $p\in ((k-1)\tau,t)$. Thus, we have
\begin{flalign*}
&G_i(t)\\
\leq&(\eta\beta+1)f_i(t-1-(k-1)\tau)\\
&+\eta\beta\sum_{j=1}^{t-(k-1)\tau}\gamma^{j}f_i(t-j-(k-1)\tau)+\eta\delta_{i}\sum_{j=0}^{t-(k-1)\tau}\gamma^{j}\\
&\rightline{\text{(from \eqref{ieq:Gitleq}, \eqref{ieq:Gip} and $G_i((k-1)\tau)=f_i(0)$)}}\\
=&(\eta\beta+1)\left(\gamma^{t-1-(k-1)\tau}g_i(t-1-(k-1)\tau)-\frac{\delta_i}{\beta}\right)\\
&+\eta\beta\sum_{j=1}^{t-(k-1)\tau}\left(\gamma^{t-(k-1)\tau}g_i(t-j-(k-1)\tau)-\gamma^{j}\frac{\delta_i}{\beta}\right)\\
&+\eta\delta_i \sum_{j=0}^{t-(k-1)\tau}\gamma^j\\
&\rightline{\text{(from \eqref{eq:fix})}}\\
=&\gamma^{t-1-(k-1)\tau}\left((\eta\beta+1)g_i(t-1-(k-1)\tau)\right.\\
&\left.+\eta\beta\gamma\sum_{j=1}^{t-(k-1)\tau}g_i(t-j-(k-1)\tau)\right)-\frac{\delta_i}{\beta}\\
=&\gamma^{t-(k-1)\tau}g_i(t-(k-1)\tau)-\frac{\delta_i}{\beta}\\
&\rightline{\text{(from Lemma \ref{lemma:sequence} and $g_i(t)=a_t$)}}\\
=&f_i(t-(k-1)\tau).
\end{flalign*}

Thus, Lemma \ref{lemma:w_it-w_kt} has been proven.
\end{proof}
\subsection{Bounding \texorpdfstring{$\Vert\mathbf{w}(t)-\mathbf{w}_{[k]}(t) \Vert$}{||w(t)-w[k](t)||}}
Based on the result of Lemma \ref{lemma:w_it-w_kt}, we first bound the gap of $\Vert\mathbf{v}(t)-\mathbf{v}_{[k]}(t) \Vert$ in Lemma \ref{lemma:vt-vkt}. Based on the result of Lemma \ref{lemma:vt-vkt}, we then bound the gap of $\Vert\mathbf{w}(t)-\mathbf{w}_{[k]}(t) \Vert$, which concludes Theorem \ref{theorem:wt-wkt}.

\begin{lemma} \label{lemma:vt-vkt}
For any interval $[k]$, $\forall t \in [(k-1)\tau,k\tau]$, we have:
\begin{align}
    &\Vert\mathbf{v}(t)-\mathbf{v}_{[k]}(t) \Vert\nonumber\\ \leq&\eta\delta\left(\frac{C(\gamma A)^{t_0}}{\gamma(A-1)}+\frac{D(\gamma B)^{t_0}}{\gamma(B-1)}-\frac{\gamma^{t_0}-1}{\gamma-1}\right),&
\end{align}
where $t_0=t-(k-1)\tau$.
\end{lemma}

\begin{proof}[Proof of Lemma \ref{lemma:vt-vkt}]
For convenience, we define
\begin{align} \label{eq:pt=}
    p(t) \triangleq \gamma^t(CA^t+DB^t)-1.
\end{align}
Therefore, we get
\begin{align} \label{eq:fit=pt}
    f_i(t) = \frac{\delta_i}{\beta}p(t).
\end{align}
From \eqref{eq:localVi} and \eqref{eq:globalV}, we have
\begin{align} \label{eq:vt=}
    \mathbf{v}(t)=\gamma\mathbf{v}(t-1)-\eta\frac{\sum_{i=1}^N D_i \nabla F_i(\mathbf{w}_i (t-1))}{D}.
\end{align}
For $t \in ((k-1)\tau, k\tau]$, we have
\begin{flalign} \label{eq:vt-vkt}
&\Vert\mathbf{v}(t)-\mathbf{v}_{[k]}(t) \Vert&\nonumber\\
=&\Vert \gamma\mathbf{v}(t-1)-\eta\frac{\sum_{i=1}^N D_i \nabla F_i(\mathbf{w}_i (t-1))}{D}&\nonumber\\
&-\gamma\mathbf{v}_{[k]}(t-1)+\eta\nabla F(\mathbf{w}_{[k]} (t-1)) \Vert&\nonumber\\
&\rightline{\text{(from \eqref{eq:vt=} and \eqref{eq:vkt=})}}&\nonumber\\
\leq&\gamma\Vert\mathbf{v}(t-1)-\mathbf{v}_{[k]}(t-1)\Vert&\nonumber\\
&+\eta\frac{\sum_{i=1}^N D_i\Vert \nabla F_i(\mathbf{w}_i(t-1))- \nabla F_i(\mathbf{w}_{[k]}(t-1))\Vert}{D}&\nonumber\\
\leq&\gamma\Vert\mathbf{v}(t-1)-\mathbf{v}_{[k]}(t-1)\Vert&\nonumber\\
&+\eta\beta\frac{\sum_{i=1}^N D_i f_i(t-1-(k-1)\tau)}{D}&\nonumber\\
&\rightline{\text{(from $\beta$-smoothness and Lemma \ref{lemma:w_it-w_kt})}}&\nonumber\\
=&\gamma\Vert\mathbf{v}(t-1)-\mathbf{v}_{[k]}(t-1)\Vert+\eta\delta p(t-1-(k-1)\tau).&\\
&\rightline{\text{(from \eqref{eq:fit=pt} and \eqref{eq:delta=})}}&\nonumber
\end{flalign}
We use $\gamma^0, \gamma^1,\dots, \gamma^{t-(k-1)\tau-1}$ as multipliers to multiply \eqref{eq:vt-vkt} when $t, t-1,\dots, (k-1)\tau+1$, respectively.
\begin{flalign*}
&\Vert \mathbf{v}(t) - \mathbf{v}_{[k]}(t) \Vert&\\ 
\leq& \gamma\Vert \mathbf{v}(t-1) - \mathbf{v}_{[k]}(t-1) \Vert+ \eta\delta p(t-1-(k-1)\tau),&\\
&\gamma\Vert \mathbf{v}(t-1) - \mathbf{v}_{[k]}(t-1) \Vert&\\ 
\leq&\gamma^2(\Vert\mathbf{v}(t-2) - \mathbf{v}_{[k]}(t-2)\Vert+\gamma\eta\delta p(t-2-(k-1)\tau),&\\ 
&\dots& \\
&\gamma^{t-(k-1)\tau-1}\Vert \mathbf{v}((k-1)\tau+1)-\mathbf{v}_{[k]}((k-1)\tau+1)\Vert&\\ 
\leq& \gamma^{t-(k-1)\tau}\Vert\mathbf{v}((k-1)\tau)-\mathbf{v}_{[k]}((k-1)\tau)\Vert&\\
&+\gamma^{t-1-(k-1)\tau}\eta\delta p(0).&
\end{flalign*}
Summing up all of the above inequalities, we have
\begin{flalign} 
&\Vert \mathbf{v}(t) - \mathbf{v}_{[k]}(t) \Vert\leq\eta\delta\sum_{j=1}^{t-(k-1)\tau}\gamma^{t-j-(k-1)\tau}p(j-1)&\label{ieq:vt-vkt1}\\
&\rightline{\text{(because $\Vert\mathbf{v}((k-1)\tau) - \mathbf{v}_{[k]}((k-1)\tau)\Vert=0$ from \eqref{eq:vk(k-1)tau})}}&\nonumber\\
=&\eta\delta\left(\gamma^{t-1-(k-1)\tau}C\sum_{j=1}^{t-(k-1)\tau}A^{j-1}\right.&\nonumber\\
&\left.+\gamma^{t-1-(k-1)\tau}D\sum_{j=1}^{t-(k-1)\tau}B^{j-1}-\sum_{j=1}^{t-(k-1)\tau}\gamma^{j-1}\right)&\nonumber\\
=&\eta\delta\left(\gamma^{t_0-1}C\frac{A^{t_0}-1}{A-1}+\gamma^{t_0-1}D\frac{B^{t_0}-1}{B-1}-\frac{\gamma^{t_0}-1}{\gamma-1}\right)&\nonumber\\
=&\eta\delta\left(\frac{C(\gamma A)^{t_0}}{\gamma(A-1)}+\frac{D(\gamma B)^{t_0}}{\gamma(B-1)}-\frac{\gamma^{t_0}-1}{\gamma-1}\right)&\nonumber\\
&-\eta\delta\gamma^{t_0-1}\left(\frac{C}{A-1}+\frac{D}{B-1}\right)&\nonumber\\
=&\eta\delta\left(\frac{C(\gamma A)^{t_0}}{\gamma(A-1)}+\frac{D(\gamma B)^{t_0}}{\gamma(B-1)}-\frac{\gamma^{t_0}-1}{\gamma-1}\right)&\label{ieq:vt-vkt}
\end{flalign}
where $t_0=t-(k-1)\tau$.
Thus, Lemma \ref{lemma:vt-vkt} has been proven.
\end{proof}

Based on the result in Lemma \ref{lemma:vt-vkt}, we can now bound $\Vert\mathbf{w}(t)-\mathbf{w}_{[k]}(t)\Vert$. 
\begin{proof} [Proof of Theorem \ref{theorem:wt-wkt}]
From \eqref{eq:localWi}, \eqref{eq:globalV}, and \eqref{eq:globalW}, we have
\begin{align} \label{eq:wt=}
\mathbf{w}(t)=\mathbf{w}(t-1)+\gamma\mathbf{v}(t)-\eta\frac{\sum_{i=1}^N D_i \nabla F_i(\mathbf{w}_i (t-1))}{D}.
\end{align}
From \eqref{eq:wkt=} and \eqref{eq:wt=}, we have
\begin{flalign*}
&\Vert\mathbf{w}(t)-\mathbf{w}_{[k]}(t) \Vert&\nonumber\\
=&\Vert \mathbf{w}(t-1)+\gamma\mathbf{v}(t)-\eta\frac{\sum_{i=1}^N D_i \nabla F_i(\mathbf{w}_i (t-1))}{D}&\nonumber\\
&-\mathbf{w}_{[k]}(t-1)-\gamma\mathbf{v}_{[k]}(t)+\eta\nabla F(\mathbf{w}_{[k]} (t-1)) \Vert&\nonumber\\
\leq&\Vert\mathbf{w}(t-1)-\mathbf{w}_{[k]}(t-1) \Vert +\gamma\Vert\mathbf{v}(t)-\mathbf{v}_{[k]}(t) \Vert&\nonumber\\
&+\eta\delta p(t-1-(k-1)\tau).&\\
&\rightline{\text{(from $\beta$-smoothness, Lemma \ref{lemma:w_it-w_kt}, \eqref{eq:fit=pt}, and \eqref{eq:delta=})}}&\nonumber
\end{flalign*}
Thus, according to Lemma \ref{lemma:vt-vkt}, we have
\begin{flalign}
&\Vert\mathbf{w}(t)-\mathbf{w}_{[k]}(t) \Vert - \Vert\mathbf{w}(t-1)-\mathbf{w}_{[k]}(t-1) \Vert&\nonumber\\
\leq& \gamma\eta\delta\left(\frac{C(\gamma A)^{t_0}}{\gamma(A-1)}+\frac{D(\gamma B)^{t_0}}{\gamma(B-1)}-\frac{\gamma^{t_0}-1}{\gamma-1}\right)&\nonumber\\
&+\eta\delta(\gamma^{t_0-1}(CA^{t_0-1}+DB^{t_0-1})-1)&\label{ieq:wt-wkt_gap1}\\
=&\eta\delta\left(\frac{C(\gamma A)^{t_0-1}}{A-1}(\gamma A+A-1)\right.&\nonumber\\
&\left.+\frac{D(\gamma B)^{t_0-1}}{B-1}(\gamma B+B-1)-\frac{\gamma^{t_0+1}-1}{\gamma-1}\right).&\label{ieq:wt-wkt_gap}
\end{flalign}
When $t=(k-1)\tau$, we have $\Vert \mathbf{w}(t)-\mathbf{w}_{[k]}(t) \Vert =0$. When $t\in((k-1)\tau,k\tau]$, we sum up \eqref{ieq:wt-wkt_gap} for $t, t-1,\dots, (k-1)\tau+1$. Then we have
\begin{flalign*}
&\Vert\mathbf{w}(t)-\mathbf{w}_{[k]}(t) \Vert&\\
\leq&\sum_{x=1}^{t_0}\eta\delta\left(\frac{C(\gamma A)^{x-1}}{A-1}(\gamma A+A-1)\right.&\\
&\left.+\frac{D(\gamma B)^{x-1}}{B-1}(\gamma B+B-1)-\frac{\gamma^{x+1}-1}{\gamma-1}\right)&\\
=&\eta \delta\left[E\left((\gamma A)^{t_{0}}-1\right)+F\left((\gamma B)^{t_{0}}-1\right)\right.&\\
&\left.-\frac{\gamma^2(\gamma^{t_{0}}-1)-(\gamma-1) t_{0}}{(\gamma-1)^{2}}\right]&\\
=&\eta \delta\left[E(\gamma A)^{t_{0}}+F(\gamma B)^{t_{0}}-\frac{1}{\eta \beta}-\frac{\gamma^2(\gamma^{t_{0}}-1)-(\gamma-1) t_{0}}{(\gamma-1)^{2}}\right]&\\
=&h(t_0),&
\end{flalign*}
where $E=\frac{\gamma A+A-1}{(A-B)(\gamma A-1)}$ and  $F=\frac{\gamma B+B-1}{(A-B)(1-\gamma B)}$ (as defined in Theorem~\ref{theorem:wt-wkt}).   $E+F=\frac{1}{\eta\beta}$. $t_0=t-(k-1)\tau$.  Thus,  Theorem \ref{theorem:wt-wkt} has been proven.
\end{proof}

%% file: appendixC.tex
\section{Proof of Monotone of \texorpdfstring{$h(x)$}{h(x)} (Observation \texorpdfstring{\textcircled{1}}{1} in Theorem \ref{theorem:wt-wkt})} \label{appendixC}
We first introduce following Lemma~\ref{lemma:ABCD} for later use.
\begin{lemma} \label{lemma:ABCD}
Given $A,B,C,$ and $D$ according to their definitions, then we have
\begin{align*}
    C(\gamma A)^i+D(\gamma B)^i \geq (1+\eta\beta+\eta\beta\gamma)^i
\end{align*}
holds for $i=0,1,2,3,...$
\end{lemma}
\begin{proof}
We note that according to the definitions of $A,B,C$ and $D$, we know that $\gamma A >1, 0<\gamma B<1,\frac{1}{\gamma +1}<B<1, C>0, D>0,E>0,$ and $F>0$. We also have $C+D=1$. 

When $i=0, C(\gamma A)^i+D(\gamma B)^i = (1+\eta\beta+\eta\beta\gamma)^i =1$, so the inequality holds. When $i=1$, we have
\begin{align*}
&C(\gamma A)^i+D(\gamma B)^i\\
=& \gamma(CA+DB)\\
=& \gamma\left(\frac{A-1}{A-B}A+\frac{1-B}{A-B}B\right)\\
=&\gamma(A+B-1)\\
=&1+\eta\beta+\eta\beta\gamma,
\end{align*}
so the inequality still holds. When $i>1$, according to Jensen inequality, and $f(x)=x^i$ is convex, we have
\begin{align*}
&C(\gamma A)^i+D(\gamma B)^i\\
\geq&(\gamma CA+\gamma DB)^i\\
=&(1+\eta\beta+\eta\beta\gamma)^i.
\end{align*}
To conclude, Lemma \ref{lemma:ABCD} has been proven.
\end{proof}
Then we can prove the monotone of $h(x)$.
\begin{proof}
It is equivalent to prove
\begin{align*}
h(x)-h(x-1)\geq 0
\end{align*}
for all integer $x\geq 1$. When $x=0$ or $x=1$, we have
\begin{align*}
h(0)&=\eta\delta(E+F-\frac{1}{\eta\beta})=0,\\
h(1)&=\eta\delta\left(\gamma(EA+FB)-\frac{1}{\eta\beta}-\gamma-1\right)=0,
\end{align*}
because $EA+FB=\frac{1+\eta\beta+\eta\beta\gamma}{\eta\beta\gamma}$. Therefore, when $x=1, h(x)-h(x-1)=0$.

When $x>1$, according to Lemma \ref{lemma:ABCD} and \eqref{eq:pt=}, we have $p(x)=C(\gamma A)^x+D(\gamma B)^x -1 \geq (1+\eta\beta+\eta\beta\gamma)^x -1>0$. Then we have
\begin{flalign*}
&h(x)-h(x-1)\\
=&\eta\delta\left(\frac{C(\gamma A)^{x}(\gamma A+A-1)}{\gamma A(A-1)}+\frac{D(\gamma B)^{x}(\gamma B+B-1)}{\gamma B(B-1)}\right.\\
&\left.-\frac{\gamma^{x+1}-1}{\gamma-1}\right)\\
=&\gamma\eta\delta\left(\frac{C(\gamma A)^{x}}{\gamma(A-1)}+\frac{D(\gamma B)^{x}}{\gamma(B-1)}-\frac{\gamma^{x}-1}{\gamma-1}\right)\\
&+\eta\delta(\gamma^{x-1}(CA^{x-1}+DB^{x-1})-1)\\
&\rightline{\text{(because \eqref{ieq:wt-wkt_gap} equals \eqref{ieq:wt-wkt_gap1})}}\\
=&\gamma\eta\delta\sum_{j=1}^{x}\gamma^{x-j}p(j-1)+\eta\delta p(x-1)\\
&\rightline{\text{(because \eqref{ieq:vt-vkt} equals \eqref{ieq:vt-vkt1}, $x=t-(k-1)\tau$, and \eqref{eq:pt=})}}\\
>& 0.
\end{flalign*}

Thus, we have proven that $h(0)=h(1)=0$ and $h(x)$ increases with $x$ when $x\geq1$.
\end{proof}

%% file: appendixD.tex
\section{Proof of Theorem 2} \label{appendixD}
For convenience, we define
$c_{[k]}(t) \triangleq F(\mathbf{w}_{[k]}(t)) - F(\mathbf{w}^*)$
for a given interval $[k]$, where $t \in [(k-1)\tau,k\tau]$.
\begin{proof}
According to the convergence lower bound of any gradient descent methods given in Theorem 3.14 in \cite{bubeck2014convex}, we always have
\begin{equation} \label{ieq:ckt>0}
    c_{[k]}(t)>0
\end{equation}
for any $t$ and $k$.



Then we derive the upper bound of $c_{[k]}(t+1)-c_{[k]}(t)$, where $t \in [(k-1)\tau,k\tau-1]$. 

Because $F(\cdot)$ is $\beta$-smooth, according to Lemma 3.4 in \cite{bubeck2014convex}, we have
\begin{align*}
F(\mathbf{x})-F(\mathbf{y}) \leq \nabla F(\mathbf{y})^{\mathrm{T}}(\mathbf{x}-\mathbf{y})+\frac{\beta}{2}\|\mathbf{x}-\mathbf{y}\|^{2}
\end{align*}
for arbitrary $\mathbf{x}$ and $\mathbf{y}$. Thus,
\begin{flalign} \label{eq:ck(t+1)-ckt}
&c_{[k]}(t+1)-c_{[k]}(t)&\nonumber\\
=&F\left(\mathbf{w}_{[k]}(t+1)\right)-F\left(\mathbf{w}_{[k]}(t)\right)&\nonumber \\
\leq& \nabla F\left(\mathbf{w}_{[k]}(t)\right)^{\mathrm{T}}\left(\mathbf{w}_{[k]}(t+1)-\mathbf{w}_{[k]}(t)\right)&\nonumber \\
&+\frac{\beta}{2}\left\|\mathbf{w}_{[k]}(t+1)-\mathbf{w}_{[k]}(t)\right\|^{2}&\nonumber \\
=&\gamma\nabla F\left(\mathbf{w}_{[k]}(t)\right)^{\mathrm{T}}\mathbf{v}_{[k]}(t+1)-\eta\|\nabla F\left(\mathbf{w}_{[k]}(t)\right)\|^2&\nonumber\\
&+\frac{\beta}{2}\|\gamma\mathbf{v}_{[k]}(t+1)-\eta \nabla F\left(\mathbf{w}_{[k]}(t)\right)\|^2&\nonumber\\
=&-\eta(\gamma+1)\left(1-\frac{\beta\eta(\gamma+1)}{2}\right)\left\|\nabla F(\mathbf{w}_{[k]}(t))\right\|^2&\nonumber\\
+&\frac{\beta\gamma^4}{2}\left\|\mathbf{v}_{{[k]}}(t)\right\|^2+\gamma^2\left(1-\beta\eta(\gamma+1)
\right)\nabla F\left(\mathbf{w}_{[k]}(t)\right)^{\mathrm{T}}\mathbf{v}_{[k]}(t)&\nonumber\\
&\rightline{\text{(replacing $\mathbf{v}_{[k]}(t+1)$ with \eqref{eq:vkt=} and rearrange)}}&\nonumber\\
\leq& \left(-\eta(\gamma+1)\left(1-\frac{\beta\eta(\gamma+1)}{2}\right)+\frac{\beta\eta^2\gamma^2 p^2}{2}\right.&\nonumber\\
&\left.+\frac{(1-\beta\eta(\gamma+1))(1+\eta^2\gamma^2 p^2)}{2}\right)\|\nabla F\left(\mathbf{w}_{[k]}(t)\right)\|^{2},&
\end{flalign}
where the second term in \eqref{eq:ck(t+1)-ckt} is because $\|\gamma \mathbf{v}_{[k]}(t)\| \leq p \|\eta \nabla F(\mathbf{w}_{[k]}(t))\|$ with the definition of $p$. The third term in \eqref{eq:ck(t+1)-ckt} is because
\begin{flalign*}
&\nabla F(\mathbf{w}_{[k]}(t))^{\mathrm{T}} \mathbf{v}_{[k]}(t)\\
\leq&\frac{1}{2a}\|\nabla F(\mathbf{w}_{[k]}(t))\|^2 + \frac{a}{2} \|\mathbf{v}_{[k]}(t)\|^2\\
=&\frac{1}{2\gamma^2}\|\nabla F(\mathbf{w}_{[k]}(t))\|^2 + \frac{\gamma^2}{2} \|\mathbf{v}_{[k]}(t)\|^2\\
&\rightline{\text{(Young’s Inequality, for any $a > 0$ and we set $a=\gamma^2$)}}\\
\leq&\frac{1}{2\gamma^2}\|\nabla F(\mathbf{w}_{[k]}(t))\|^2 + \frac{p^2\eta^2}{2} \|\nabla F(\mathbf{w}_{[k]}(t))\|^2\\
=&\left(\frac{1}{2\gamma^2}+\frac{p^2\eta^2}{2}\right)\|\nabla F(\mathbf{w}_{[k]}(t))\|^2.\\
\end{flalign*}
According to the definition of $\alpha$, and condition 2 of Theorem \ref{theorem:Fwt-Fw*} with $h(\tau)\geq0$, we have $\alpha>0$. Then from \eqref{eq:ck(t+1)-ckt}, we have
\begin{align} 
\label{ieq:ckt+1<ckt}
c_{[k]}(t+1)\leq c_{[k]}(t) -\alpha\left\|\nabla F(\mathbf{w}_{[k]}(t))\right\|^{2}.
\end{align}
According to the convexity condition and Cauchy-Schwarz inequality, we have:
\begin{flalign*}
c_{[k]}(t) &=F(\mathbf{w}_{[k]}(t))-F(\mathbf{w}^{*}) \leq \nabla F(\mathbf{w}_{[k]}(t))^{\mathrm{T}}(\mathbf{w}_{[k]}(t)-\mathbf{w}^{*}) \\
& \leq\left\|\nabla F(\mathbf{w}_{[k]}(t))\right\|\left\|\mathbf{w}_{[k]}(t)-\mathbf{w}^{*}\right\|.
\end{flalign*}
Equivalently,
\begin{align} \label{ieq:delta_Fwt>}
\left\|\nabla F(\mathbf{w}_{[k]}(t))\right\| \geq \frac{c_{[k]}(t)}{\left\|\mathbf{w}_{[k]}(t)-\mathbf{w}^{*}\right\|}.
\end{align}
Substituting \eqref{ieq:delta_Fwt>} into \eqref{ieq:ckt+1<ckt}, and noting $\omega\leq\frac{1}{\left\|\mathbf{w}_{[k]}(t)-\mathbf{w}^{*}\right\|^{2}}$ by the definition of $\omega$, we get
\begin{flalign*}
c_{[k]}(t+1)\leq& c_{[k]}(t) - \frac{\alpha c_{[k]}(t)^2}{\left\|\mathbf{w}_{[k]}(t)-\mathbf{w}^{*}\right\|^2}\\
\leq& c_{[k]}(t) - \omega\alpha c_{[k]}(t)^2.
\end{flalign*}
Because $\alpha >0$, $c_{[k]}(t)>0$ in \eqref{ieq:ckt>0}, and \eqref{ieq:ckt+1<ckt}, we have $0<c_{[k]}(t+1)\leq c_{[k]}(t)$. Dividing both side by $c_{[k]}(t+1)c_{[k]}(t)$, we get
\begin{align*}
\frac{1}{c_{[k]}(t)} \leq \frac{1}{c_{[k]}(t+1)}-\omega\alpha\frac{c_{[k]}(t)}{c_{[k]}(t+1)}.
\end{align*}
We note that $\frac{c_{[k]}(t)}{c_{[k]}(t+1)} \geq 1$. Thus,
\begin{align}\label{itrt:ckt}
\frac{1}{c_{[k]}(t+1)}-\frac{1}{c_{[k]}(t)}\geq \omega\alpha\frac{c_{[k]}(t)}{c_{[k]}(t+1)}\geq \omega\alpha.
\end{align}
Summing up the above inequality by $t\in [(k-1)\tau, k\tau-1]$, we have
\begin{flalign}
&\frac{1}{c_{[k]}(k\tau)}-\frac{1}{c_{[k]}((k-1)\tau)}\nonumber\\
=&\sum_{t=(k-1)\tau}^{k\tau-1}\left(\frac{1}{c_{[k]}(t+1)}-\frac{1}{c_{[k]}(t)}\right)\nonumber\\
\geq& \sum_{t=(k-1)\tau}^{k\tau-1}\omega\alpha = \tau\omega\alpha.
\end{flalign}
Then, we sum up the above inequality by $k\in [1,K]$, after rearranging the left-hand side and noting that $T=K\tau$, we can get
\begin{flalign}\label{ieq:cKT-c10}
&\sum_{k=1}^{K}\left(\frac{1}{c_{[k]}(k\tau)}-\frac{1}{c_{[k]}((k-1)\tau)}\right)\nonumber\\
=&\frac{1}{c_{[K]}(T)}-\frac{1}{c_{[1]}(0)}-\sum_{k=1}^{K-1}\left(\frac{1}{c_{[k+1]}(k\tau)}-\frac{1}{c_{[k]}(k\tau)}\right)\nonumber\\
\geq& K\tau\omega\alpha = T\omega\alpha.
\end{flalign}
Here, we note that
\begin{flalign} \label{ieq:frac:c(k+1)-ck}
\frac{1}{c_{[k+1]}(k\tau)}-\frac{1}{c_{[k]}(k\tau)}=&\frac{c_{[k]}(k\tau)-c_{[k+1]}(k\tau)}{c_{[k]}(k\tau)c_{[k+1]}(k\tau)}\nonumber\\
=&\frac{F(\mathbf{w}_{[k]}(k\tau))-F(\mathbf{w}_{[k+1]}(k\tau))}{c_{[k]}(k\tau)c_{[k+1]}(k\tau)}\nonumber\\
\geq& \frac{-\rho h(\tau)}{c_{[k]}(k\tau)c_{[k+1]}(k\tau)}.
\end{flalign}
where the last inequality is because $\mathbf{w}_{[k+1]}(k\tau)=\mathbf{w}(k\tau)$ in \eqref{eq:wk(k-1)tau}, and \eqref{ieq:Fwt-Fwkt} in Theorem \ref{theorem:wt-wkt}.

From \eqref{ieq:ckt+1<ckt},  we can get $F(\mathbf{w}_{[k]}(t))\geq F(\mathbf{w}_{[k]}(t+1))$ for any $t\in [(k-1)\tau,k\tau)$. Recalling condition 3 in Theorem \ref{theorem:Fwt-Fw*}, where $F(\mathbf{w}_{[k]}(k \tau))-F\left(\mathbf{w}^{*}\right) \geq \varepsilon$ for all $k$, we can obtain $c_{[k]}(t)= F(\mathbf{w}_{[k]}(t)) - F(\mathbf{w}^*) \geq \varepsilon$ for all $t\in [(k-1)\tau,k\tau]$ and $k$. Thus,
\begin{align} \label{ieq:ck*c(k+1)}
c_{[k]}(k\tau)c_{[k+1]}(k\tau)\geq\varepsilon^2.
\end{align}
According to Appendix C, we have $h(\tau)\geq0$. Then substituting \eqref{ieq:ck*c(k+1)} into \eqref{ieq:frac:c(k+1)-ck}, we have
\begin{align} \label{ieq:frac:c(k+1)-ck_cont.}
\frac{1}{c_{[k+1]}(k\tau)}-\frac{1}{c_{[k]}(k\tau)} \geq\frac{-\rho h(\tau)}{\varepsilon^2}.
\end{align}
Substituting \eqref{ieq:frac:c(k+1)-ck_cont.} into \eqref{ieq:cKT-c10} and rearrange, we get
\begin{align}\label{ieq:frac:cKT-c10}
\frac{1}{c_{[K]}(T)}-\frac{1}{c_{[1]}(0)}\geq T\omega\alpha-(K-1)\frac{\rho h(\tau)}{\varepsilon^2}.
\end{align}
Recalling condition 4 in Theorem \ref{theorem:Fwt-Fw*}, where $F(\mathbf{w}(T))-F(\mathbf{w}^{*}) \geq \varepsilon$, and noting that $c_{[K]}(T)\geq\varepsilon$, we get
\begin{align}
(F(\mathbf{w}(T))-F(\mathbf{w}^{*}))c_{[K]}(T)\geq \varepsilon^2
\end{align}
Thus,
\begin{flalign} \label{ieq:frac:FwT-cKT}
&\frac{1}{F(\mathbf{w}(T))-F\left(\mathbf{w}^{*}\right)}-\frac{1}{c_{[K]}(T)}\nonumber\\
=&\frac{c_{[K]}(T)-(F(\mathbf{w}(T))-F\left(\mathbf{w}^{*}\right))}{(F(\mathbf{w}(T))-F\left(\mathbf{w}^{*}\right))c_{[K]}(T)}\nonumber\\
=&\frac{F(\mathbf{w}_{[K]}(T))-F(\mathbf{w}(T))}{(F(\mathbf{w}(T))-F\left(\mathbf{w}^{*}\right))c_{[K]}(T)}\nonumber\\
\geq& \frac{-\rho h(\tau)}{(F(\mathbf{w}(T))-F\left(\mathbf{w}^{*}\right))c_{[K]}(T)}\nonumber\\
\geq& -\frac{\rho h(\tau)}{\varepsilon^2},
\end{flalign}
where the first inequality is because \eqref{ieq:Fwt-Fwkt} in Theorem \ref{theorem:wt-wkt} when $t=K\tau$ in interval $[K]$. Combining \eqref{ieq:frac:cKT-c10} with \eqref{ieq:frac:FwT-cKT}, we get
\begin{flalign*}
\frac{1}{F(\mathbf{w}(T))-F\left(\mathbf{w}^{*}\right)}-\frac{1}{c_{[1]}(0)}\geq& T\omega\alpha-K\frac{\rho h(\tau)}{\varepsilon^2}\\
=&T\omega\alpha-\frac{T\rho h(\tau)}{\tau\varepsilon^2}\\
=&T\left(\omega\alpha-\frac{\rho h(\tau)}{\tau\varepsilon^2}\right).
\end{flalign*}
Noting that $c_{[1]}(0)=F(\mathbf{w}_{[1]}(0))-F(\mathbf{w}^*)>0$, the above inequality can be expressed as
\begin{align}
\frac{1}{F(\mathbf{w}(T))-F\left(\mathbf{w}^{*}\right)}\geq T\left(\omega\alpha-\frac{\rho h(\tau)}{\tau\varepsilon^2}\right).
\end{align}
Recalling condition 2 in Theorem \ref{theorem:Fwt-Fw*}, where $\omega\alpha-\frac{\rho h(\tau)}{\tau \varepsilon^{2}}>0$, we obtain that the right-hand side of above inequality is greater than zero. Therefore, taking the reciprocal of the above inequality, we finally get the result
\begin{align*}
F(\mathbf{w}(T))-F\left(\mathbf{w}^{*}\right)\leq \frac{1}{T\left(\omega\alpha-\frac{\rho h(\tau)}{\tau\varepsilon^2}\right)}.
\end{align*}
\end{proof}

%% file: appendixE.tex
\section{Proof of Theorem 3} \label{appendixE}
\begin{proof}

At the beginning, we see that condition 1 in Theorem \ref{theorem:Fwt-Fw*} always holds due to the conditions in Theorem \ref{theorem:Fwf-Fw*}, where $0<\beta\eta(\gamma+1)\leq 1$, and $0\leq\gamma< 1$.

When $\rho h(\tau)=0$, there is always an arbitrarily small $\varepsilon$ but great than zero that let conditions 2--4 in Theorem \ref{theorem:Fwt-Fw*} hold. Under this circumstance, Theorem \ref{theorem:Fwt-Fw*} holds. We also note that the right-hand side of \eqref{ieq:Fwf-Fw*} is equivalent to the right-hand side of \eqref{ieq:FwT-Fw*} when $\rho h(\tau)=0$. Moreover, according to the definition of $\mathbf{w}^{\mathrm{f}}$ in \eqref{eq:wf}, we have
\begin{align*}
F\left(\mathbf{w}^{\mathrm{f}}\right)-F\left(\mathbf{w}^{*}\right) \leq F(\mathbf{w}(T))-F\left(\mathbf{w}^{*}\right) \leq \frac{1}{T\omega\alpha},
\end{align*}
which satisfies the result in Theorem \ref{theorem:Fwt-Fw*} directly. Thus, Theorem \ref{theorem:Fwf-Fw*} holds when $\rho h(\tau)=0$.

When $\rho h(\tau)>0$, considering the right-hand side of \eqref{ieq:FwT-Fw*} and let
\begin{align} \label{eq:varepsilon_0}
\varepsilon_0= \frac{1}{T\left(\omega\alpha-\frac{\rho h(\tau)}{\tau\varepsilon_0^2}\right)}.
\end{align}
Rearranging and calculating $\varepsilon_0$, we get
\begin{align}\label{eq:varepsilon_0_cont.}
\varepsilon_{0}=\frac{1}{2 T \omega \alpha}+\sqrt{\frac{1}{4 T^{2} \omega^{2} \alpha^{2}}+\frac{\rho h(\tau)}{\omega \alpha \tau}}.
\end{align}
Here, we take the positive solution $\varepsilon_0$ because $\varepsilon >0$ in Theorem \ref{theorem:Fwt-Fw*}. Considering above two equations for $\varepsilon_0$, we get $\varepsilon_0 > 0$ and the denominator in \eqref{eq:varepsilon_0} is greater than zero. We also note that $\omega\alpha-\frac{\rho h(\tau)}{\tau\varepsilon^2}$ increases with $\varepsilon$. Thus, when $\varepsilon \geq \varepsilon_0$, condition 2 in Theorem \ref{theorem:Fwt-Fw*} holds. Under this circumstance, we assume that there exists $\varepsilon > \varepsilon_0$ that satisfies both condition 3 and 4 in Theorem \ref{theorem:Fwt-Fw*} at the same time, so that Theorem \ref{theorem:Fwt-Fw*} holds. Then we get,
\begin{flalign*}
F(\mathbf{w}(T))-F\left(\mathbf{w}^{*}\right) \leq& \frac{1}{T\left(\omega \alpha-\frac{\rho h(\tau)}{\tau \varepsilon^{2}}\right)}\\
<& \frac{1}{T\left(\omega\alpha-\frac{\rho h(\tau)}{\tau\varepsilon_0^2}\right)} = \varepsilon_0,
\end{flalign*}
which contradicts the condition 4 in Theorem \ref{theorem:Fwt-Fw*}. Using the proof by contradiction, we conclude that there does not exist $\varepsilon>\varepsilon_0$ that satisfies both condition 3 and 4 in Theorem \ref{theorem:Fwt-Fw*} at the same time. Equivalently, it happens either (1) $\exists k\in [1,K]$ allows $F\left(\mathbf{w}_{[k]}(k \tau)\right)-F\left(\mathbf{w}^{*}\right) \leq \varepsilon_0$ or (2) $F(\mathbf{w}(T))-F\left(\mathbf{w}^{*}\right) \leq \varepsilon_0$, which follows
\begin{align}\label{ieq:min}
\min \left\{\min _{k\in [1,K]} F\left(\mathbf{w}_{[k]}(k \tau)\right) ; F(\mathbf{w}(T))\right\}-F\left(\mathbf{w}^{*}\right) \leq \varepsilon_{0}.
\end{align}
Recalling \eqref{ieq:Fwt-Fwkt} in Theorem \ref{theorem:wt-wkt}, when $t=k\tau$, we have $F(\mathbf{w}(k\tau)) \leq F(\mathbf{w}_{[k]}(k\tau)) + \rho h(\tau)$ for any interval $[k]$. 
Combining it with \eqref{ieq:min}, we have
\begin{align*}
\min _{k\in [1,K]} F(\mathbf{w}(k \tau))-F\left(\mathbf{w}^{*}\right) \leq \varepsilon_{0}+\rho h(\tau).
\end{align*}
Recalling the definition of $\mathbf{w}^{\mathrm{f}}$ in \eqref{eq:wf}, $T=K\tau$, and combining $\mathbf{w}^{\mathrm{f}}$ with above inequality, we get
\begin{align*}
F\left(\mathbf{w}^{\mathrm{f}}\right)-F\left(\mathbf{w}^{*}\right) \leq \varepsilon_{0}+\rho h(\tau).
\end{align*}
Substituting \eqref{eq:varepsilon_0_cont.} into above inequality, we finally get the result in \eqref{ieq:Fwf-Fw*}, which proves the Theorem \ref{theorem:Fwf-Fw*}.
\end{proof}

%% file: appendixF.tex
\section{Proof of Theorem 4}\label{appendixF}
\begin{proof}
When $\eta \to 0$, we have $\gamma A \simeq 1$, $\gamma B \simeq \gamma$, and $F \simeq \frac{\gamma^2}{(1-\gamma)^2}$. Therefore,
\begin{flalign*}
&\lim_{\eta\to0}{h(\tau)}&\\
=&\lim_{\eta\to0}{\eta \delta\left[E(\gamma A)^{\tau}+F(\gamma B)^{\tau}-\frac{1}{\eta \beta}-\frac{\gamma^2(\gamma^{\tau}-1)-(\gamma-1) \tau}{(\gamma-1)^{2}}\right]}&\\
=&\lim_{\eta\to0}{\eta\delta\left(E-\frac{1}{\eta\beta}\right)}&\\
=&\lim_{\eta\to0}{\eta\delta\left(\frac{1}{(1-\gamma)(\gamma A-1)}-\frac{1}{\eta\beta}\right)}&\\
=&\frac{\delta}{1-\gamma}\lim_{\eta\to0}{\frac{\eta}{\gamma A-1}}-\frac{\delta}{\beta}&\\
=&\frac{\delta}{1-\gamma}\lim_{\eta\to0}{\frac{1}{(\gamma A-1)^{'}}}-\frac{\delta}{\beta}&\\
=&\frac{\delta}{1-\gamma}\frac{1-\gamma}{\beta}-\frac{\delta}{\beta}=0
\end{flalign*}
where the second last line is because the L'Hôpital's rule.
We also have $\hat{h}(\tau)\simeq0$ when $\eta\to 0$. Rewrite $f_1(T)$ and $f_2(T)$, we have
\begin{align*}
f_1(T)&=\frac{1}{2 T \omega \alpha}+\sqrt{\frac{1+4T^2\omega\alpha\rho h(\tau)\tau^{-1}}{4T^{2}\omega^{2}\alpha^{2}}}+\rho h(\tau)&\\
&\simeq\frac{1}{2 T \omega \alpha}+\sqrt{\frac{1}{4T^{2}\omega^{2}\alpha^{2}}}=\frac{1}{T\omega\alpha},\\
f_2(T)&=\frac{1}{2 T \omega\hat{\alpha}}+\sqrt{\frac{1+4T^2\omega\hat{\alpha}\rho \hat{h}(\tau)\tau^{-1}}{4T^{2}\omega^{2}\hat{\alpha}^{2}}}+\rho \hat{h}(\tau)&\\
&\simeq\frac{1}{2 T \omega\hat{\alpha}}+\sqrt{\frac{1}{4T^{2}\omega^{2}\hat{\alpha}^{2}}}=\frac{1}{T\omega\hat{\alpha}}.
\end{align*}
According to the definition of $\alpha$, and condition 2 of Theorem~\ref{theorem:Fwt-Fw*} with $h(\cdot)\geq0$, we have $\alpha>0$. Based on the conditions in Theorem  \ref{theorem:fast}, where $0<\beta\eta(\gamma+1)\leq 1$ and the definition of $\hat{\alpha}$, we have $\hat{\alpha}>0$. Furthermore, for any $0<\gamma<1$ and $\eta\to 0^+$, we have $\alpha>\hat{\alpha}$. Therefore, we get $f_1(T)<f_2(T)$.
\end{proof}